\documentclass[letterpaper]{article} 
\usepackage{aaai2026}  
\usepackage{times}  
\usepackage{helvet}  
\usepackage{courier}  
\usepackage[hyphens]{url}  
\usepackage{graphicx} 
\usepackage{natbib}  
\usepackage{caption} 
\usepackage{algorithm}
\usepackage{algorithmic}

\usepackage{amsmath,amssymb,amsfonts} 
\usepackage{amsthm,mathrsfs}
\usepackage{bm}

\usepackage{graphicx} 
\usepackage[font=small,skip=2pt]{caption} 
\usepackage{subcaption}
\usepackage{booktabs,tabularx,multirow,makecell,diagbox}
\usepackage{tabularray}

\usepackage[title]{appendix}
\usepackage{xcolor}
\usepackage{textcomp}
\usepackage{manyfoot}
\usepackage[switch]{lineno}
\usepackage{placeins}

\theoremstyle{thmstyleone}
\newtheorem{theorem}{Theorem}

\newtheorem{lemma}{Lemma}
\newtheorem{Corollary}{Corollary}

\theoremstyle{thmstyletwo}

\theoremstyle{thmstylethree}
\newtheorem{definition}{Definition}
\newtheorem{Property}{Property}

\usepackage{newfloat}
\usepackage{listings}
\DeclareCaptionStyle{ruled}{labelfont=normalfont,labelsep=colon,strut=off} 
\floatstyle{ruled}
\newfloat{listing}{tb}{lst}{}
\floatname{listing}{Listing}
\title{RI-Loss: A Learnable Residual-Informed Loss for Time Series Forecasting}
\author {
    Jieting Wang\textsuperscript{\rm 1,2},
    Xiaolei Shang\textsuperscript{\rm 1},
    Feijiang Li\textsuperscript{\rm 1,2},
    Furong Peng\textsuperscript{\rm 1,2}
}
\affiliations {
    \textsuperscript{\rm 1}Institute of Big Data Science and Industry, Shanxi University\\
    \textsuperscript{\rm 2}Key Laboratory of Evolutionary Science Intelligence of Shanxi Province, Taiyuan, Shanxi, China\\
}

\usepackage{bibentry}

\begin{document}

\maketitle

\begin{abstract}

Time series forecasting relies on predicting future values from historical data, yet most state-of-the-art approaches—including transformer and multilayer perceptron-based models—optimize using Mean Squared Error (MSE), which has two fundamental weaknesses: its point-wise error computation fails to capture temporal relationships, and it does not account for inherent noise in the data. To overcome these limitations, we introduce the Residual-Informed Loss (RI-Loss), a novel objective function based on the Hilbert-Schmidt Independence Criterion (HSIC). RI-Loss explicitly models noise structure by enforcing dependence between the residual sequence and a random time series, enabling more robust, noise-aware representations. Theoretically, we derive the first non-asymptotic HSIC bound with explicit double-sample complexity terms, achieving optimal convergence rates through Bernstein-type concentration inequalities and Rademacher complexity analysis. This provides rigorous guarantees for RI-Loss optimization while precisely quantifying kernel space interactions. Empirically, experiments across eight real-world benchmarks and five leading forecasting models demonstrate improvements in predictive performance, validating the effectiveness of our approach. The code is publicly available at: https://github.com/shang-xl/RI-Loss.

\end{abstract}


\section{Introduction}

Time series data consist of chronologically ordered observations, where forecasting involves identifying latent patterns and trends from historical data to predict future values. In practice, temporal datasets typically demonstrate considerable complexity, manifesting as nonlinear relationships, non-stationary distributions, long-term dependencies, and significant noise contamination. Deep learning architectures excel at modeling these intricate temporal structures due to their exceptional capacity for nonlinear function approximation, allowing for more precise characterization of dynamic temporal patterns ~\cite{Cao2018BRITSBR,Liu2021SCINetTS,2021GraphNN}. These advanced modeling capabilities have enabled transformative applications across multiple domains, including financial market prediction \cite{Informer}, industrial process monitoring \cite{li2023r}, disease spread forecasting~\cite{Matsubarakdd14,Li2025}, and meteorological condition estimation \cite{Angryk2020MultivariateTS}.

Recent advances in time series analysis have witnessed the emergence of transformer-based and multilayer perceptron (MLP)-based architectures as dominant approaches, with numerous studies demonstrating their state-of-the-art performance \cite{autoformer,zeng2022,xu2024fits,Wang2023rss,liu2024itransformer}. However, despite their widespread adoption of MSE as the standard loss function, this conventional choice presents two critical limitations for long-term forecasting tasks.

The first limitation of MSE lies in its point-wise error computation, which makes it susceptible to overfitting observational noise, thereby impairing model generalization. More critically, MSE’s formulation disregards the temporal dynamics of sequential data, neglecting essential interdependencies across time steps. These inherent weaknesses hinder the model’s capacity to capture and preserve long-range dependencies effectively.

To better capture temporal dependencies, we propose a novel loss function inspired by a fundamental principle from \cite{timeseries2019}: an ideal model should produce residuals that are statistically indistinguishable from random noise, indicating complete extraction of all predictable patterns. While directly quantifying the difference between residuals and random sequences is problematic due to noise's inherent stochasticity, we reformulate this intuition through statistical independence measures. Specifically, our approach maximizes the dependence between model residuals and random noise sequences, thereby compelling the model to extract more informative temporal features from the input data.

Building on these insights, we propose a novel forecasting framework that incorporates kernel-based nonparametric independence testing via a residual-informed loss function. Our approach centers on the Hilbert-Schmidt Independence Criterion (HSIC), which provides a theoretically sound measure of statistical dependence while remaining computationally efficient. The framework achieves two key innovations: it effectively captures nonlinear temporal dependencies between model residuals and random noise through kernel methods, and it establishes a new dependency-aware modeling paradigm that unifies statistical learning principles with deep neural forecasting architectures.

Theoretically, we establish rigorous learning guarantees by analyzing the self-bounding properties of HSIC and deriving novel generalization bounds through double-sample Rademacher complexity analysis. These theoretical contributions achieve dual significance: they formally validate the learnability of HSIC-Loss, and they provide a generalizable analytical framework for assessing the generalization capacity of dependence-aware loss functions. Our comprehensive empirical evaluation demonstrates consistent performance gains across diverse real-world datasets.

The main contributions of this paper are as follows:
\begin{itemize}
    \item Methodological Innovation: We propose RI-Loss, a novel residual-informed loss function that simultaneously addresses two fundamental limitations of MSE: susceptibility to observational noise and failure to capture temporal dependencies, establishing a new kernel-based learning paradigm for time series forecasting.

    \item Theoretical Advancement: We develop the first second-order Rademacher complexity bounds for the Hilbert-Schmidt Independence Criterion (HSIC) in time series analysis, providing rigorous generalization guarantees and creating a theoretical framework for analyzing dependence-aware loss functions.

    \item Empirical Validation: Through extensive experiments on eight diverse real-world benchmarks, we demonstrate that RI-Loss consistently enhances performance across state-of-the-art architectures (including Transformers and MLPs).
\end{itemize}
Complete proofs and additional experimental results are provided in the Appendix.

\section{Related Work}

\subsection{Time Series Forecasting}
Time series forecasting has undergone significant evolution, progressing from simple moving averages to classical statistical methods like Autoregressive Integrated Moving Average (ARIMA)~\cite{Box1970}, through machine learning approaches such as Support Vector Regression (SVR)~\cite{Smola2004}, to contemporary deep learning architectures. The deep learning revolution has introduced specialized models for temporal data processing: Recurrent Neural Networks (RNNs) effectively capture sequential dependencies through iterative time-step processing~\cite{Cao2018BRITSBR}; Convolutional Neural Networks (CNNs) extract localized temporal patterns via convolutional kernels~\cite{Hewage20,Liu2021SCINetTS}; Graph Neural Networks (GNNs) model complex multivariate interactions~\cite{2021GraphNN,Cao20}; Transformers utilize self-attention mechanisms to identify both long-range dependencies and cross-variable correlations~\cite{Informer,autoformer,liu2024itransformer}; while Multilayer Perceptrons (MLPs) provide computationally efficient solutions for resource-constrained scenarios~\cite{li2023r,zeng2022,xu2024fits}. Despite these architectural advances, most deep time series models continue to rely on Mean Squared Error (MSE) loss for parameter optimization, presenting limitations in handling noise and temporal relationships.

Recent work has proposed MSE alternatives via shape alignment \cite{DTW,leguen19dilate} and dependency modeling \cite{wang2025fredf,li2025k}, but these face computational complexity or residual pattern neglect. Our RI-Loss addresses both issues through statistical independence, offering noise-robust temporal modeling via kernel dependency measurement, and provable learning guarantees through HSIC-based theory.

\subsection{HSIC-based Learning Method}

The Hilbert-Schmidt Independence Criterion (HSIC) has emerged as a powerful kernel-based independence measure across multiple machine learning domains. In representation learning, it enhances latent variable interpretability through deep generative model regularization~\cite{li2021selfsupervised,Wang2020learning}. For causal inference, HSIC enables nonlinear conditional independence testing, overcoming linearity constraints in causal discovery~\cite{Li2024PHSICAR,Robert24}. Transfer learning benefits from its synergistic use with maximum mean discrepancy for cross-domain feature alignment~\cite{Wang2020}, while time series analysis leverages HSIC for nonlinear Granger causality detection in financial modeling~\cite{REN2020123245,Wang2025stabilizing}. Despite these successful applications, HSIC's potential for time series forecasting remains largely untapped, offering significant opportunities to advance temporal dependency modeling.

\section{Preliminaries}

\subsection{Time Series Forecasting Task}
%

Time series forecasting aims to predict future values based on historical observations. In our framework, given a lookback window $\boldsymbol{X}_t = (X_t, X_{t-1}, ..., X_{t-w+1})^\top \in \mathbb{R}^w$ of $w$ historical observations, we predict the $H$-step future trajectory $\boldsymbol{Y} = (Y_{t+1}, \dots, Y_{t+H})^\top \in \mathbb{R}^H$ through
$\hat{\boldsymbol{Y}} = f(\boldsymbol{X}_t)$,
where $f:\mathbb{R}^w \to \mathbb{R}^H$ is the forecasting function. For multivariate series with $d$-dimensional observations $\boldsymbol{x}_t \in \mathbb{R}^d$, this extends to predicting $\boldsymbol{Y} = {\boldsymbol{y}_{t+1}, \dots, \boldsymbol{y}_{t+H}} \in \mathbb{R}^{H \times d}$ from input $\boldsymbol{X}_t = {\boldsymbol{x}_t, \boldsymbol{x}_{t-1}, ..., \boldsymbol{x}_{t-w+1}} \in \mathbb{R}^{w \times d}$.

Generally, the forecasting quality is evaluated via the MSE:
\begin{equation}
\mathrm{MSE} = \frac{1}{H}\|\boldsymbol{Y} - \hat{\boldsymbol{Y}}\|_F^2 = \frac{1}{H}\sum_{k=1}^H \|\boldsymbol{y}_{t+k} - \hat{\boldsymbol{y}}_{t+k}\|_2^2,
\end{equation}
where $\| \cdot\|_F^2$ represents the square of the Frobenius norm for matrices.
The optimal solution of MSE corresponds to the conditional expectation of the true data, which means minimizing MSE guarantees that the mean of the predicted sequence matches that of the true sequence.

\subsection{Additive Noise Model for Time Series Observations}
Conventional time series prediction frameworks typically adopt an additive model structure of the form:
$\boldsymbol{Y} = h(\boldsymbol{X}_t) + \boldsymbol{\epsilon},$
where $\boldsymbol{Y}$ denotes the observation sequence, $h:\mathbb{R}^w \to \mathbb{R}^H$ represents the true underlying system dynamics that maps historical patterns to future values, and $\boldsymbol{\epsilon}$ is a zero-mean additive noise vector satisfying $\mathbb{E}[\boldsymbol{\epsilon}] = \boldsymbol{0}$. The noise process $\boldsymbol{\epsilon} = (\epsilon_{t+1},\ldots,\epsilon_{t+H})^\top$ further exhibits two key properties: (i) temporal uncorrelation with $\mathbb{E}[\epsilon_{t+i}\epsilon_{t+j}] = 0$ for all $i,j=1,\ldots,H$ where $i \neq j$, and (ii) signal independence with $\epsilon_{t+i} \perp h(\boldsymbol{X}_t)$ for all prediction horizons $i$.

Although we define the loss within the above basic framework, it naturally generalizes to more complex scenarios, including heteroscedastic noise structures $\epsilon_{t+i} \sim \mathcal{N}(0,\sigma_{t+i}^2)$ that accommodate time-varying volatility patterns commonly observed in financial and environmental time series, as well as state-dependent noise models of the form $Y_{t+i} = h(\boldsymbol{X}_t) + g(\boldsymbol{X}_t)\epsilon_{t+i}$ where the noise amplitude $g(\boldsymbol{X}_t)$ varies with the system state.

\subsection{Hilbert-Schmidt Independence Criterion}

The Hilbert-Schmidt Independence Criterion (HSIC) provides a kernel-based measure of statistical dependence between random variables. Formally, for variables $R \in \mathcal{R}$ and $S \in \mathcal{S}$ with joint distribution $\mathbb{P}_{R,S}$, using reproducing kernel Hilbert spaces (RKHS) $\mathcal{F}$ on $\mathcal{R}$ and $\mathcal{G}$ on $\mathcal{S}$, the HSIC is defined as the squared Hilbert-Schmidt norm of their cross-covariance operator.

As established by \cite{Gretton2005,Greenfeld}, the HSIC admits an important probabilistic interpretation: it represents the maximum achievable covariance between normalized functions in the respective RKHS. Mathematically, for $f \in \mathcal{F}$ and $g \in \mathcal{G}$, this is expressed as:
\begin{equation}
\mathbb{E}(\mathrm{HSIC}(R,S)) = \sup_{\substack{\|f\|_{\mathcal{F}} \leq 1, \|g\|_{\mathcal{G}} \leq 1}} (\mathbb{COV}[f(R),g(S)])^2,
\label{supcov}
\end{equation}
where $\mathbb{COV}$ denotes covariance. The HSIC possesses two key properties:
$\text{HSIC}(R,S) = 0$ if and only if $R$ and $S$ are independent, and
the magnitude of HSIC reflects the strength of dependence.
This formulation makes HSIC particularly powerful for detecting nonlinear dependencies that conventional correlation measures might miss, as it operates in high-dimensional feature spaces through the kernel trick.

By expanding the Hilbert-Schmidt norm through kernel-based inner products, we derive the following equivalent formulation of HSIC:

\begin{definition}[Population HSIC]
\label{def:kernel_hsic}
For random variables $R \in \mathcal{R}$ and $S \in \mathcal{S}$ with joint distribution $\mathbb{P}_{RS}$, let $k: \mathcal{R} \times \mathcal{R} \to \mathbb{R}$ and $l: \mathcal{S} \times \mathcal{S} \to \mathbb{R}$ be characteristic kernels, and let $(R',S')$ be an independent copy of $(R,S)$. The Hilbert-Schmidt Independence Criterion admits the following equivalent expression:
\begin{align}
\mathrm{HSIC}(R,S) &= \mathbb{E}_{RR'SS'}[k(R,R')l(S,S')] \notag \\[0.5ex]
&\quad + \mathbb{E}_{RR'}[k(R,R')]\mathbb{E}_{SS'}[l(S,S')] \notag \\[0.5ex]
&\quad - 2\mathbb{E}_{RS}\left[\mathbb{E}_{R'}[k(R,R')]\mathbb{E}_{S'}[l(S,S')]\right].
\end{align}
\end{definition}

The HSIC statistic has range $[0, \infty)$. This formulation reveals that HSIC measures the discrepancy between the joint embedding and the product of marginal embeddings.

The kernel-based approach provides three key advantages: it is nonparametric (making no assumptions about the dependence structure), flexible (accommodating mixed data types through kernel choice), and universal (detecting any measurable dependence when using characteristic kernels).

We now present the finite-sample estimator of HSIC, which operationalizes the population measure for practical computation.
\begin{definition}[Empirical HSIC Estimator]
\label{def:empirical_hsic}
Given an i.i.d. sample $\{(r_i,s_i)\}_{i=1}^n$, the U-statistic estimator of HSIC is given by:
\begin{align}
&\mathrm{HSIC} = \frac{1}{\binom{n}{2}} \sum_{i<j} k(r_i,r_j)l(s_i,s_j) \nonumber \\
& + \frac{1}{\binom{n}{4}} \sum_{i<j<k<l} k(r_i,r_j)l(s_k,s_l) \nonumber \\
& - \frac{2}{\binom{n}{3}} \sum_{i<j<k} \big[k(r_i,r_j)l(s_j,s_k) + k(r_j,s_k)l(s_i,s_j)\big],
\end{align}
where all sums are over distinct index tuples, the binomial coefficient $\binom{n}{m}$ represents the number of ways to choose $m$ elements from a set of $n$ distinct items, and $k,l$ are characteristic kernels.
\end{definition}

\section{Residual-Informed Loss}

\subsection{Decomposition of MSE Loss for Additive Model}
Under the additive noise model described earlier, where the observed outputs $\bm{Y}$ are corrupted by additive noise $\boldsymbol{\epsilon}$ with mean zero such that $\bm{Y} = h(\bm{X}_t) + \bm{\epsilon}$, we can analyze the relationship between the true MSE and the observed MSE. The true MSE measures the prediction error with respect to the noise-free target $h(\bm{X}_t)$:
$\mathrm{MSE}_{\mathrm{true}} = \frac{1}{H}\|h(\bm{X}_t) - \hat{\bm{Y}}\|_F^2.$
The observed MSE quantifies the discrepancy between predictions and noisy measurements:
$\mathrm{MSE}_{\mathrm{obs}} = \frac{1}{H}\|\bm{Y} - \hat{\bm{Y}}\|_F^2.$

Using the observation model $h(\bm{X}_t) = \bm{Y} - \bm{\epsilon}$, we can express the true MSE as:
\begin{align}
\small
\mathrm{MSE}_{\mathrm{true}} &= \frac{1}{H}\|(\bm{Y} - \bm{\epsilon}) - \hat{\bm{Y}}\|_F^2 \nonumber \\
&= \frac{1}{H}\left( \|\bm{Y} - \hat{\bm{Y}}\|_F^2 - 2\langle \bm{Y} - \hat{\bm{Y}}, \bm{\epsilon} \rangle + \|\bm{\epsilon}\|_F^2 \right) \nonumber \\
&= \mathrm{MSE}_{\mathrm{obs}} - \frac{2}{H}\langle \bm{Y} - \hat{\bm{Y}}, \bm{\epsilon} \rangle + \frac{1}{H}\|\bm{\epsilon}\|_F^2.
\label{dec}
\end{align}

Assuming the noise $\bm{\epsilon}$ is zero-mean and uncorrelated with the prediction error, taking expectations on both sides yields:
\begin{equation}
\mathbb{E}[\mathrm{MSE}_{\mathrm{true}}] = \mathbb{E}[\mathrm{MSE}_{\mathrm{obs}}] - \frac{1}{H}\mathbb{E} [\|\bm{\epsilon}\|_F^2],
\end{equation}
where $\mathbb{E} [\|\bm{\epsilon}\|_F^2]$ denotes the noise variance like term. This shows that the observed MSE systematically overestimates the true prediction error by exactly the noise power.
However, in general cases, since the prediction $\hat{\bm{Y}}$ is learned from noisy observations $\bm{Y}$, the prediction error $\hat{\bm{Y}}-\bm{Y}$ depends on the noisy $\bm{\epsilon}$, making the cross-term in Eq. (\ref{dec}) non-zero:

\begin{theorem}[Cross-Term Expectation for Linear Projection]
\label{Cross-Term-Expectation}
Let $\bm{Y} = (Y_{t+1}, \dots, Y_{t+H})^\top \in \mathbb{R}^H$ be a noisy observation vector generated by:
$\bm{Y} = h(\bm{X}_t) + \bm{\epsilon},$
where $h(\bm{X}_t): \mathcal{X} \to \mathbb{R}^H$ is a deterministic mapping, and the noise vector $\bm{\epsilon} \in \mathbb{R}^H$ satisfies:
$\mathbb{E}[\bm{\epsilon}] = \bm{0},
\mathbb{E}[\bm{\epsilon}\bm{\epsilon}^\top] = \sigma^2\bm{I}_H,
\bm{\epsilon} \sim \mathcal{N}(\bm{0}, \sigma^2\bm{I}_H).$
Here, $\bm{I}_H$ denotes the $H\times H$ identity matrix and $\sigma^2 > 0$ is the noise variance.
For any linear estimator $\hat{\bm{Y}} = \bm{P}\bm{Y}$ with projection matrix $\bm{P} \in \mathbb{R}^{H \times H}$, the expected normalized cross-term evaluates to:
\begin{equation}
\mathbb{E}\left[\frac{1}{H}\langle \bm{Y} - \hat{\bm{Y}}, \bm{\epsilon} \rangle\right] = \frac{\sigma^2}{H}\mathrm{tr}(\bm{I}_H - \bm{P}),
\label{crossterm}
\end{equation}
where $\mathrm{tr}(\cdot)$ represents the trace of a matrix, defined as the sum of its diagonal elements.
\end{theorem}
The decomposition of MSE in Eq.(\ref{dec}) and Theorem \ref{Cross-Term-Expectation} motivate our design of a loss function that explicitly accounts for the dependency between prediction residuals and inherent noise. To explicitly penalize both prediction accuracy and residual-noise dependency, we first propose a naive loss combining the empirical MSE with a linear residual-noise alignment term:
\(
\mathcal{L}_{\mathrm{naive}} = \mathrm{MSE}_{\mathrm{obs}} - \lambda \langle \bm{Y} - \hat{\bm{Y}}, \bm{\epsilon} \rangle,
\)
where $\lambda$ is the trade-off parameter. However, this linear term only captures first-order dependencies.

\subsection{Definition of Residual-Informed Loss}
To generalize $\langle \bm{Y} - \hat{\bm{Y}}, \bm{\epsilon} \rangle$ to nonlinear relationships, we replace the inner product with the HSIC.
We analyze their relationship.
If \( R \) and \( S \) are centered random variables, that is, \( \mathbb{E}[R] = 0 \) and \(\mathbb{E}[S] = 0 \), then their covariance simplifies to $\mathbb{COV}(R, S) = \mathbb{E}[RS]$.
At this time, we can observe that \( \mathbb{E}[RS] \leq \mathbb{E}(\mathrm{HSIC}(R,S)) \). This inequality holds because
if \(\mathcal{F}\) and \(\mathcal{G}\) contain the identity functions (i.e., \( f(R)=R \in \mathcal{F} \), \( g(S)=S \in \mathcal{G} \)), then
$
\mathbb{E}(\mathrm{HSIC}(R,S))=\sup_{f \in \mathcal{F}, g \in \mathcal{G}} \mathbb{COV}[f(R),g(S)] \geq \mathbb{COV}[R,S].
$ Thus,
assuming the residual expectation is zero ($\mathbb{E}[\bm{Y} - \hat{\bm{Y}}] = \bm{0}$) and noise expectation is zero ($\mathbb{E}[\bm{\epsilon}] = \bm{0}$), we have $\mathbb{E}(\langle \bm{Y} - \hat{\bm{Y}}, \bm{\epsilon} \rangle )
\leq\mathbb{E}( \mathrm{HSIC}(\bm{Y} - \hat{\bm{Y}}, \bm{\epsilon}))$. This
yields the Residual-Informed (RI) loss:
\begin{definition}
For the observed output series $\bm{Y}$, the predicted series $\hat{\bm{Y}}$ and a noise series $\boldsymbol{\epsilon}$, RI loss is defined as:
\begin{equation}
\mathrm{RI}(\bm{Y},\hat{\bm{Y}}) = \mathrm{MSE}_{\mathrm{obs}} + \lambda \exp ( -\tau\cdot\mathrm{HSIC}(\bm{Y} - \hat{\bm{Y}}, \bm{\epsilon})),
\label{eq:RI}
\end{equation}
where the exponential transformation is applied to ensure that the loss function remains positive-valued and $\lambda$, $\tau$ are the trade-off parameter and transformation parameter, respectively.
\end{definition}


RI-Loss in Eq. (\ref{eq:RI})
requires the model to preserve interpretable signals by minimizing the MSE to ensure prediction accuracy,
while explicitly separating the noise structure by maximizing the dependence between the residuals and the noise, thereby forcing the model to push unexplainable variation (noise structure) into the residuals.
Mathematically, when the residual \( \bm{Y} - \hat{\bm{Y}} \) is highly correlated with the noise \(\bm{\epsilon}\), it indicates that the model has excluded as much of the noise-distributed components as possible from the prediction \( \hat{\bm{Y}} \), ensuring that the prediction result \( \hat{\bm{Y}} \) contains only the true signal that is independent of the noise.

\subsection{RI Advantage}

We investigate how RI loss varies with increasing noise ratios $\rho \in [0,1]$, comparing its behavior to conventional MSE. The results is demonstrated in Figure \ref{fig:trader-off}.
Our experiment generates a sinusoidal signal with noise $\bm{y}_{\text{true}} = \sin(\bm{x}) + \bm{\epsilon}$ where $\bm{x} \in [0,2\pi]$ contains 1000 points and $\bm{\epsilon} \sim \mathcal{N}(0,1)$ represents baseline Gaussian noise. Crucially, for each $\rho$ value (51 linearly spaced steps), we corrupt the underlying noise-free sinusoidal component by adding $\mathcal{N}(0,1)$ noise to randomly selected $\rho \times 100\%$ of the $\sin(\bm{x})$ values, producing $\bm{y}_{\text{noisy}}$. The metrics include $\mathrm{MSE} = \mathbb{E}[(\bm{y}_{\text{true}} - \bm{y}_{\text{noisy}})^2]$ and $\mathrm{RI} = \mathrm{MSE} + \exp(-\tau \cdot \mathrm{HSIC}(\bm{y}_\triangle,\bm{\epsilon}))$ ($\tau={50, 100}$, $\lambda=1$), where $\bm{y}_\triangle= \bm{y}_{\text{true}} - \bm{y}_{\text{noisy}}$. For HSIC computation, we employ Gaussian kernel matrices $k_{ij} = \exp(-|\bm{y}_{\triangle,i} -\bm{y}_{\triangle,j}|^2)$ for residual and $l_{ij} = \exp(-|\bm{\epsilon}_i - \bm{\epsilon}_j|^2)$ for noise.

As shown in Fig.~\ref{fig:trader-off}, the RI loss demonstrates a convex dependence on the noise ratio $\rho$, reflecting the fundamental trade-off between signal reconstruction and noise suppression. For about $\rho < 0.5$ in the left panel ($\rho < 0.6$ in the right panel), the system prioritizes signal fitting, where controlled noise retention improves reconstruction fidelity (manifested in decreasing MSE). When about $\rho > 0.5$ in the left panel ($\rho > 0.6$ in the right panel), the exponential HSIC penalty ($\exp(-\tau\cdot\mathrm{HSIC}) \to 1$) becomes dominant, enforcing noise rejection at the expense of elevated RI loss. The characteristic minimum at $\rho \approx 0.5$ in the left panel ($\rho \approx  0.6$ in the right panel) naturally emerges as the optimal operating point that balances these competing requirements.

 \begin{figure}[!htbp]
\centering
\vspace{-3mm}  
\includegraphics[width=0.21\textwidth,height=2.5cm]{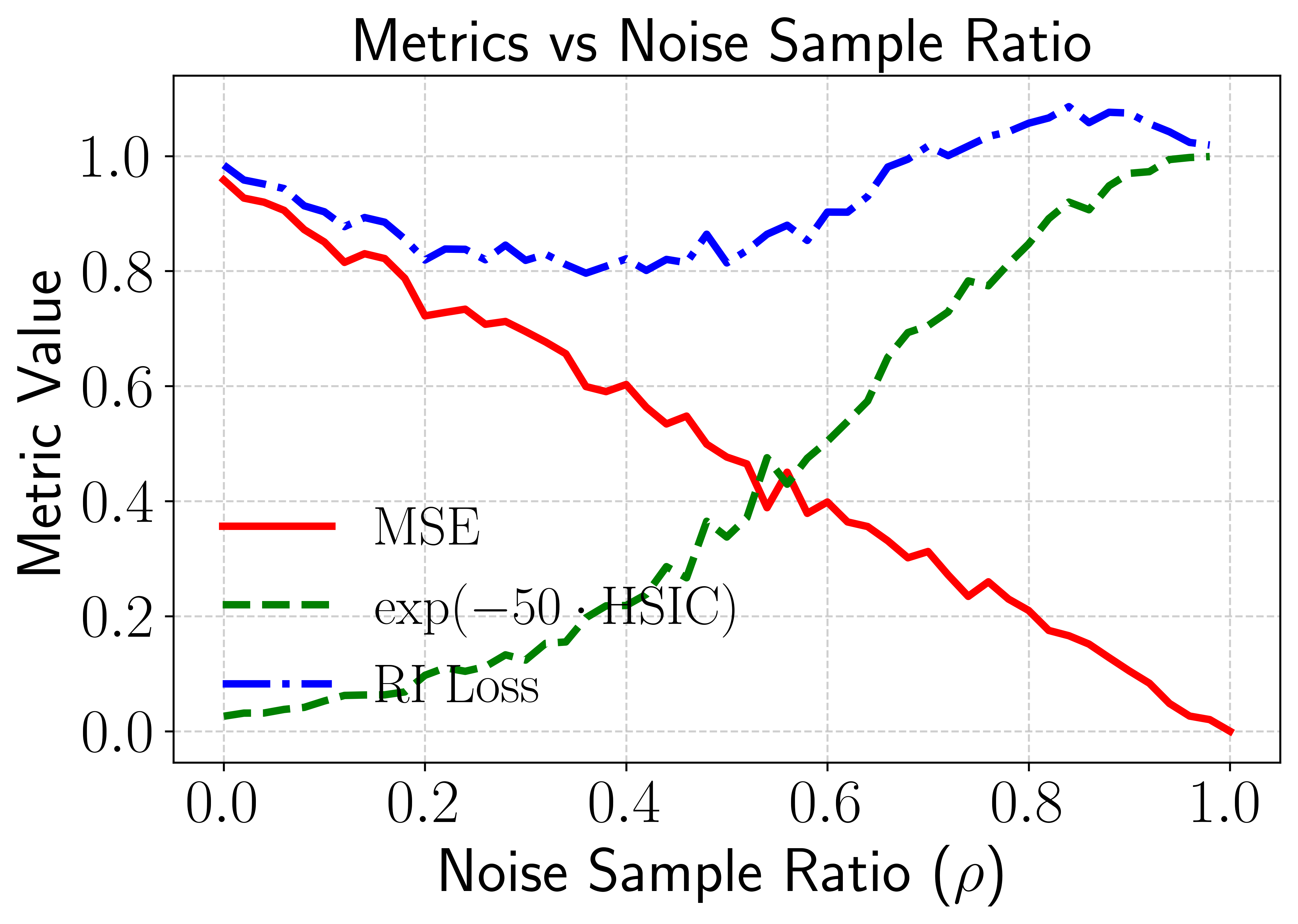}
\includegraphics[width=0.21\textwidth,height=2.5cm]{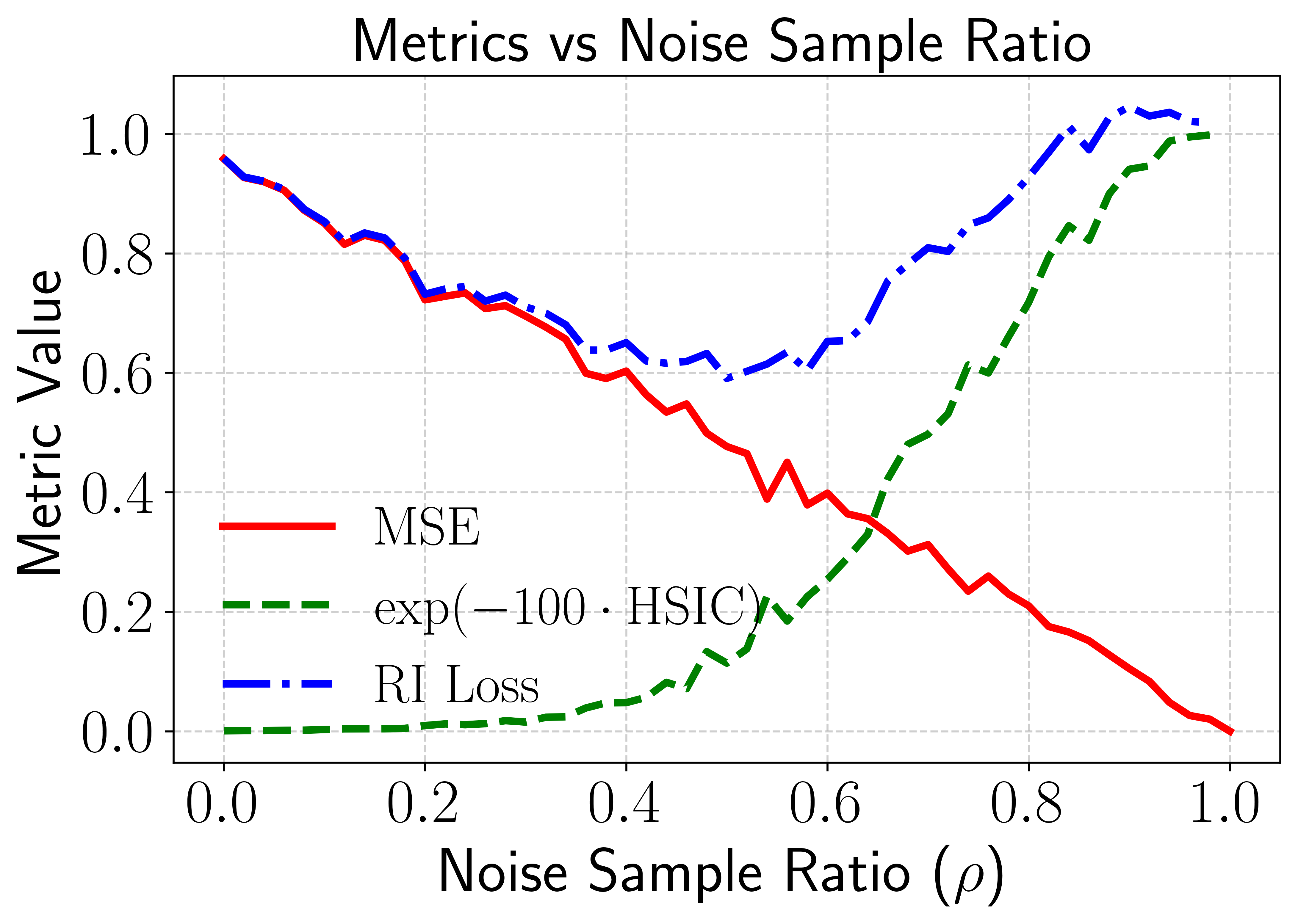}
\caption{The RI loss varies with the noise ratio.}
\label{fig:trader-off}
\vspace{-8mm}  
\end{figure}

\section{Theoretical Analysis of RI}
The major purpose of this paper is presenting a novel approach that enhances standard MSE optimization by incorporating a HSIC regularization term between model residuals and random noise. Hence, we establish a theoretical framework to investigate HSIC's generalization properties.

The current theoretical analysis operates under the baseline assumption of independent and identically distributed (i.i.d.) residuals and noise terms. This initial configuration provides crucial insights into HSIC's generalization properties.
In future work, we will investigate the generalization properties of both HSIC and MSE under more complex dependency structures.

The generalization gap measures how well empirical training error approximates the true expected error~\cite{Bartlett2001Rademacher}. A smaller gap indicates better generalization performance from training data to the underlying distribution. Theoretical bounds on this gap ensure reliable model behavior on unseen data. Our theorem shows that the empirical HSIC converges to the population HSIC as $n \to \infty$, providing generalization guarantees.

\begin{theorem}
\small
\label{hsicfinal}
Let $\sigma_1,...,\sigma_n$ are independent Rademacher random variables ($\mathbb{P}(\sigma_i = \pm 1) = \frac{1}{2}$), where $n$ is the number of samples.
For random variables $R \in \mathcal{R}$ and $S \in \mathcal{S}$ with joint distribution $\mathbb{P}_{RS}$, let $k: \mathcal{R} \times \mathcal{R} \to \mathbb{R}$ and $l: \mathcal{S} \times \mathcal{S} \to \mathbb{R}$ be characteristic kernels. Let $(R',S')$ be an independent and identically distributed (i.i.d.) copy of $(R,S)$.
Suppose that $ c_1(k) \leq k \leq c_2(k)$ and $ c_1(l) \leq l \leq c_2(l)$,
for $\delta>0$,
then the following holds:
\begin{align}
& \big|\mathrm{HSIC}(\{(r_i,s_i)\}_{i=1}^n)  - \mathbb{E}[\mathrm{HSIC}(R,S)]\big|  \\\notag
& \leq 3\sum_{j=k,l} c_2(j)\sum_{m=1}^3 \gamma_m(j).
\end{align}
Specifically, the above key functions for $f \in \{k,l\}$ are defined as:
\begin{align}
\label{zn2}
& \gamma_1(n,\delta;f) \\   \notag
&\quad=\frac{10(c_2(f)-c_1(f))}{n}  \ln \frac{2}{\delta}    \\   \notag
& \gamma_2(n,\delta,R_{\sigma};f)    \\   \notag
& \quad=2(c_2(f)-c_1(f))  \sqrt{
   2\ln \frac{2}{\delta} \bigg(\frac{\mathcal{R}_{\sigma}  }{2n}  + \frac{1}{n^2} \ln \frac{2}{\delta} \bigg)  } \\   \notag
& \gamma_3(n,\delta,W_{\sigma,\sigma},W_{\sigma,\alpha},W_{\sigma},F;f)   \\   \notag
&\quad=4 \bigg( \ln \frac{2}{\delta} \bigg( \frac{C_0}{ n-1 }
 \big( \mathbb{E}W_{\sigma,\sigma}
 +\sqrt{2} \mathbb{E}W_{\sigma,\alpha}
 \\ \notag
 & +\frac{2 ( \mathbb{E}W_{\sigma}+F) }{n}
+ \frac{\sqrt{8} F  n^{1/2}}{n^2} +  \frac{4F}{n^2}  \big)
 +\frac{1 }{n^2}  \ln\frac{2}{\delta} \bigg) \bigg)^{\frac{1}{2}}.
\end{align}
In the above functions, the first-order and second-order sample Rademacher complexities are:
\begin{align}
\mathcal{R}_{\sigma} &=\mathbb{E}_\sigma\left[ \sup_{f \in \mathcal{G}} \frac{1}{n} \sum_{i=1}^n \sigma_i f_1(X_i) \right], \\
W_{\sigma,\sigma}&=\sup_{f\in\mathcal{F}}\frac{1}{n^2}\bigg|\sum_{i,j}\sigma_{i}\sigma_{j}f_2\left(X_{i},X_{j}\right)\bigg|, \quad \\
 W_{\sigma,\alpha}&=\sup_{f\in\mathcal{F}}\sup_{\alpha:\left\|\alpha\right\|_{2}\leq1}\frac{1}{n^2}\sum_{i,j}\sigma_{i}\alpha_{j}f_2\left(X_{i},X_{j}\right),\\
 W_{\sigma}&=\sup_{f\in\mathcal{F},k=1,\ldots,n}\frac{1}{n}\bigg|\sum_{i=1}^{n}\sigma_{i}f_2\left(X_{i},X_{k}\right)\bigg|,\\ \quad
 F&=sup_{f\in \mathcal{F} } \left \| f_2 \right \| _{\infty },
\end{align}
where $f_1$ and $f_2$ respectively represent the centered kernels:
\begin{align}
f_1(X_i) &= \mathbb{E}[f(X_i,X)|X_i] - \mathbb{E}[f(X_i,X_j)], \\ \nonumber
f_2(X_i,X_j) &= f(X_i,X_j) - f_1(X_i) - f_1(X_j) \\ \nonumber
&\quad \quad - \mathbb{E}[f(X_i,X_j)].
\end{align}
\end{theorem}

Theorem \ref{hsicfinal} makes a significant breakthrough in the non-asymptotic analysis of HSIC statistics, establishing for the first time a convergence upper bound that explicitly incorporates double-sample Rademacher complexities. Through an innovative decomposition of the error into three terms with clear statistical significance: the $\gamma_1$ term controls the deviation with typical $\mathcal{O}(1/n)$ convergence rate ($\mathcal{O}(\cdot)$ characterizes the asymptotic order); the $\gamma_2$ term's square root structure reflects classical sub-exponential convergence properties; and the most innovative $\gamma_3$ term employs sophisticated probability inequality techniques to constrain the influence of double-sample complexities (including $W_{\sigma,\sigma}$ and $W_{\sigma,\alpha}$) at the $\mathcal{O}(1/\sqrt{n})$ level.

This theoretical achievement not only maintains optimal convergence rates but, more importantly, rigorously quantifies higher-order interaction effects in kernel function spaces in HSIC analysis, providing a new theoretical framework for high-dimensional nonparametric independence tests and offering important guidance for kernel method selection and sample complexity estimation.
In the RI loss framework, $r$ denotes the residual vector $\bm{Y}- \bm{\hat{Y}}$, $s$ denotes the noise vector $\bm{\epsilon}$, and $n$ is the sequence length. Theorem~\ref{hsicfinal} establishes that the empirical $\mathrm{HSIC}(r, s)$ estimator exhibits finite-sample convergence with error decaying as $\mathcal{O}(1/\sqrt{n})$, dependent on kernel complexities $\mathcal{R}_{\sigma}$,$W_{\sigma,\sigma}$,$ W_{\sigma,\alpha}$,$ W_{\sigma}$, and achieves asymptotic consistency:
$\lim_{n\to\infty}  {\mathrm{HSIC}}_n(r,s) \overset{a.s.}{=} \mathbb{E}_{R,S}[\mathrm{HSIC}(R,S)]$.
This result provides fundamental guarantees for RI loss optimization, ensuring both theoretical soundness and practical reliability in large-sample regimes.

\section{RI-based Time Series Forecasting}

Our proposed RI-Loss demonstrates broad compatibility with existing deep time series architectures, as illustrated in Figure~\ref{fig_riloss}. The modular design enables seamless integration with both Transformer-based (e.g., Autoformer, Informer) and MLP-based (e.g., DLinear) forecasting frameworks. 
We employ a uniform distribution over the interval $[-1,1]$ as the noise distribution. 
Complete implementation details are provided in Algorithm~1 of Appendix~B.

\vspace{-8pt}
\begin{figure}[ht]
\centering
\includegraphics[width=0.43\textwidth]{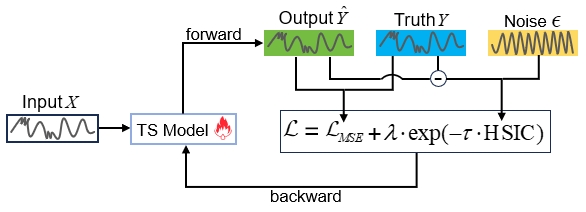}
\caption{RI-Loss Based Time Series Model.}
\label{fig_riloss}
\end{figure}
\vspace{-2em}

\section{Experimental Analysis}
%
%
We conduct a comprehensive evaluation of the proposed RI-Loss. Additional experimental results (e.g., the effect of the lookback window size and correlation structure visualization) are provided in the Appendix.

\begin{table*}[!ht]
\centering
\renewcommand{\arraystretch}{1.2}
\caption{The results of multivariate forecasting on eight datasets compared with baseline models. We use prediction lengths $H \in \{24, 36, 48, 60\}$ for ILI dataset and $H \in \{96, 192, 336, 720\}$ for the others. Ours indicates that the model uses RI-Loss as the loss function. The results with improved model performance are highlighted in bold. The last row indicates the number of top-1 times of each model.}
\label{Table_result}
\footnotesize 
\fontsize{8}{8}\selectfont
\setlength{\tabcolsep}{0.3pt} 
\begin{tabular*}{\hsize}{@{\extracolsep{\fill}}c|c|cccc|cccc|cccc|cccc|cccc}
\hline
\multicolumn{2}{c}{\centering Models}
& \multicolumn{4}{c}{Informer(2021)}
& \multicolumn{4}{c}{Autoformer(2021)}
& \multicolumn{4}{c}{Dlinear(2022)}
& \multicolumn{4}{c}{iTransformer(2024)}
& \multicolumn{4}{c}{RAFT(2025)} \\
\hline
\multicolumn{2}{c}{Metric} & MSE & Ours & MAE & Ours & MSE & Ours & MAE & Ours & MSE & Ours & MAE & Ours & MSE & Ours & MAE & Ours & MSE & Ours & MAE & Ours  \\
\hline
\multirow{4}{*}{\rotatebox[origin=c]{90}{\makecell[c]{ETTh1}}}
& 96  & 0.878 & \textbf{0.875} & 0.731 & \textbf{0.725} & 0.456 & \textbf{0.440}  & 0.458 & \textbf{0.446} & 0.384 & \textbf{0.346} & 0.405 & \textbf{0.384} &0.387	&\textbf{0.382}	&0.405 &\textbf{0.395}  & 0.387 & \textbf{0.378} & 0.414 & \textbf{0.402}
\\
& 192 & 1.013 & \textbf{0.970}  & 0.799 & \textbf{0.791} & 0.492 & \textbf{0.474} & 0.478 & \textbf{0.458} & 0.443 & \textbf{0.404} & 0.450  & \textbf{0.410} &0.441	&\textbf{0.437}	&0.436	&\textbf{0.427}  & 0.423 & \textbf{0.422} & 0.431 & \textbf{0.423}
 \\
& 336 & \textbf{1.172} & 1.225 & \textbf{0.873} & 0.888 & \textbf{0.506} & 0.522 & 0.490  & \textbf{0.484} & 0.447 & \textbf{0.443} & 0.448 & \textbf{0.439}  &0.491 &\textbf{0.490} &0.462 &\textbf{0.458}  & \textbf{0.458} & 0.460 & 0.445 & \textbf{0.439}
 \\
& 720 & 1.175 & \textbf{1.169} & 0.866 & \textbf{0.846} & \textbf{0.496} & 0.544 & \textbf{0.503} & 0.514 & 0.472 & \textbf{0.471} & 0.488 & \textbf{0.486} &0.508	&\textbf{0.490}	&0.493	&\textbf{0.477} & 0.463 & \textbf{0.462} & 0.464 & \textbf{0.458}
\\
\hline
\multirow{4}{*}{\rotatebox[origin=c]{90}{\makecell[c]{ETTh2}}}
& 96  & 3.340  & \textbf{2.299} & 1.515 & \textbf{1.221} & 0.375 & \textbf{0.336} & 0.410  & \textbf{0.375} & 0.297 & \textbf{0.279} & 0.359 & \textbf{0.332} &0.301	&\textbf{0.296}	&0.350 &\textbf{0.342} & 0.296 & \textbf{0.293} & 0.35 & \textbf{0.341}
\\
& 192 & 5.684 & \textbf{3.949} & 2.249 & \textbf{1.720}  & 0.451 & \textbf{0.417} & 0.450  & \textbf{0.425} & 0.373 & \textbf{0.355} & 0.410  & \textbf{0.386} &0.380	&\textbf{0.374}	&0.399	&\textbf{0.391} & 0.384 & \textbf{0.377} & 0.403 & \textbf{0.392}
\\
& 336 & 4.418 & \textbf{3.412} & 1.770  & \textbf{1.564} & 0.477 & \textbf{0.444} & 0.476 & \textbf{0.451} & 0.450  & \textbf{0.412} & 0.462 & \textbf{0.431}  &0.424 &\textbf{0.419} &0.432	&\textbf{0.426} & 0.426 & \textbf{0.415} & 0.442 & \textbf{0.426}
\\
& 720 & 3.249 & \textbf{3.094} & 1.528 & \textbf{1.516} & 0.476 & \textbf{0.445} & 0.489 & \textbf{0.461} & 0.696 & \textbf{0.582} & 0.595 & \textbf{0.534}  &0.430	&\textbf{0.422}	&0.447	&\textbf{0.439} & 0.434 & \textbf{0.417} & 0.460 & \textbf{0.440}
 \\
\hline
\multirow{4}{*}{\rotatebox[origin=c]{90}{\makecell[c]{ETTm1}}}
& 96  & 0.678 & \textbf{0.536} & 0.607 & \textbf{0.519} & 0.478 & \textbf{0.466} & 0.465 & \textbf{0.452} & 0.302 & \textbf{0.294} & 0.346 & \textbf{0.335}  &0.342 &\textbf{0.322} &0.377	&\textbf{0.350} & 0.329 & \textbf{0.316} & 0.371 & \textbf{0.359}
 \\
& 192 & 0.784 & \textbf{0.610}  & 0.669 & \textbf{0.569} & 0.549 & \textbf{0.534} & 0.497 & \textbf{0.476} & 0.337 & \textbf{0.333} & 0.368 & \textbf{0.358}  &0.383 &\textbf{0.376} &0.396 &\textbf{0.378} & 0.363 & \textbf{0.357} & 0.388 & \textbf{0.376}
 \\
& 336 & 1.011 & \textbf{0.807} & 0.786 & \textbf{0.684} & 0.516 & \textbf{0.480}  & 0.490  & \textbf{0.458} & 0.371 & \textbf{0.369} & 0.388 & \textbf{0.379} &0.418 &\textbf{0.412} &0.418 &\textbf{0.403} & 0.391 & \textbf{0.385} & 0.408 & \textbf{0.395}
\\
& 720 & 1.037 & \textbf{0.978} & 0.795 & \textbf{0.766} & 0.526 & \textbf{0.523} & 0.499 & \textbf{0.475} & 0.426 & \textbf{0.423} & 0.421 & \textbf{0.412} &0.487 &\textbf{0.478} &0.457 &\textbf{0.441} & \textbf{0.444} & 0.447 & 0.438 & \textbf{0.427}
\\
\hline
\multirow{4}{*}{\rotatebox[origin=c]{90}{\makecell[c]{ETTm2}}}
& 96  & 0.464 & \textbf{0.343} & 0.503 & \textbf{0.439} & 0.242 & \textbf{0.216} & 0.318 & \textbf{0.300}   & 0.170  & \textbf{0.164} & 0.262 & \textbf{0.247} &0.186 &\textbf{0.175} &0.272 &\textbf{0.252} & 0.177 & \textbf{0.174} & 0.266 & \textbf{0.257}
 \\
& 192 & 0.811 & \textbf{0.504} & 0.680  & \textbf{0.534} & 0.300   & \textbf{0.269} & 0.347 & \textbf{0.328} & 0.226 & \textbf{0.220}  & 0.302 & \textbf{0.287} &0.254 &\textbf{0.243} &0.314 &\textbf{0.299} & 0.243 & \textbf{0.238} & 0.308 & \textbf{0.299}
 \\
& 336 & 1.286 & \textbf{1.124} & 0.876 & \textbf{0.778} & 0.336 & \textbf{0.325} & 0.372 & \textbf{0.361} & 0.292 & \textbf{0.276} & 0.355 & \textbf{0.328} &0.316 &\textbf{0.306} &0.351 &\textbf{0.338} & 0.302 & \textbf{0.297} & 0.349 & \textbf{0.337}
 \\
& 720 & 3.949 & \textbf{3.269} & 1.497 & \textbf{1.353} & 0.432 & \textbf{0.415} & 0.423 & \textbf{0.411} & 0.406 & \textbf{0.375} & 0.424 & \textbf{0.395} &0.414 &\textbf{0.407} &0.407 &\textbf{0.398} & 0.402 & \textbf{0.396} & 0.404 & \textbf{0.395}
\\
\hline
\multirow{4}{*}{\rotatebox[origin=c]{90}{\makecell[c]{Exchange}}}
& 96  & 1.022 & \textbf{0.901} & 0.817 & \textbf{0.773} & 0.154 & \textbf{0.138} & 0.286 & \textbf{0.265} & 0.086 & \textbf{0.084} & 0.208 & \textbf{0.203} &0.086 &\textbf{0.085} &0.206 &\textbf{0.204} & 0.089 & \textbf{0.084} & 0.206 & \textbf{0.200}
\\
& 192 & 1.204 & \textbf{1.176} & 0.889 & \textbf{0.874} & 0.297 & \textbf{0.297} & 0.394 & \textbf{0.393} & 0.199 & \textbf{0.163} & 0.324 & \textbf{0.291} & 0.181 &\textbf{0.180} &0.304 &\textbf{0.301} & 0.192 & \textbf{0.177} & 0.309 & \textbf{0.297}
\\
& 336 & \textbf{1.577} & 1.674 & \textbf{1.001} & 1.024 & \textbf{0.457} & 0.511 & \textbf{0.501} & 0.531 & 0.345 & \textbf{0.265} & 0.445 & \textbf{0.387} &0.338 &\textbf{0.337} &0.422 &\textbf{0.421} & 0.370 & \textbf{0.343} & 0.442 & \textbf{0.423}
 \\
& 720 & \textbf{2.206} & 2.574 & \textbf{1.237} & 1.343 & \textbf{1.079} & 1.188   & \textbf{0.808} & 0.845 & \textbf{0.872} & 1.128 & \textbf{0.711} & 0.786 &\textbf{0.853} &0.863 &0.696 &\textbf{0.670} & 1.133 & \textbf{1.058} & 0.807 & \textbf{0.781}
\\
\hline
\multirow{4}{*}{\rotatebox[origin=c]{90}{\makecell[c]{ILI}}}
& 24  & 6.158 & \textbf{5.972} & 1.733 & \textbf{1.674} & 3.427 & \textbf{3.322} & 1.292 & \textbf{1.241} & 2.205 & \textbf{2.199} & 1.032 & \textbf{0.990} &\textbf{2.695} &2.739 &1.073 &\textbf{1.069} & 2.040 & \textbf{2.015} & 0.946 & \textbf{0.905}
\\
& 36 & 5.930  & \textbf{5.739} & 1.697 & \textbf{1.613} & 3.596 & \textbf{3.425} & 1.304 & \textbf{1.247} & 2.392 & \textbf{2.106} & 1.123 & \textbf{0.965} &2.563 &\textbf{2.531} &1.051 &\textbf{1.044} & 2.085 & \textbf{1.962} & 0.955 & \textbf{0.896}
\\
& 48 & \textbf{5.185} & 5.237 & 1.560  & \textbf{1.529} & 3.493 & \textbf{3.178} & 1.300   & \textbf{1.193} & 2.298 & \textbf{1.984} & 1.080  & \textbf{0.956} &2.567 &\textbf{2.455} &1.048 &\textbf{1.015} & 2.069 & \textbf{1.884} & 0.970 & \textbf{0.893}
 \\
& 60 & 5.296 & \textbf{5.091} & 1.566 & \textbf{1.516} & \textbf{2.847} & 3.019 & 1.145 & \textbf{1.150}  & 2.410  & \textbf{1.970}  & 1.115 & \textbf{0.966} &2.625 &\textbf{2.480} &1.068 &\textbf{1.029} & 2.065 & \textbf{1.859} & 0.975 & \textbf{0.902}
 \\
\hline
\multirow{4}{*}{\rotatebox[origin=c]{90}{\makecell[c]{Weather}}}
& 96  & 0.496 & \textbf{0.479} & 0.494 & \textbf{0.462} & \textbf{0.224} & \textbf{0.224}  & 0.303  & \textbf{0.285} & 0.144 & \textbf{0.143} & 0.210  & \textbf{0.179} &0.174 &\textbf{0.169} &0.214 &\textbf{0.201} & \textbf{0.189} & \textbf{0.189} & 0.246 & \textbf{0.237}
 \\
& 192 & 0.582 & \textbf{0.575} & 0.549 & \textbf{0.518} & 0.305 & \textbf{0.283} & 0.365 & \textbf{0.331} & 0.188  & \textbf{0.184} & 0.256 & \textbf{0.220} &0.224 &\textbf{0.219} &0.257 &\textbf{0.248} & 0.239 & \textbf{0.234} & 0.289 & \textbf{0.274}
 \\
& 336 & 0.643 & \textbf{0.609} & 0.590  & \textbf{0.540}  & 0.353 & \textbf{0.347} & 0.391 & \textbf{0.371} & 0.240 & \textbf{0.234} & 0.299 & \textbf{0.261} &0.283 &\textbf{0.278} &0.300 &\textbf{0.291} & 0.291 & \textbf{0.283} & 0.329 & \textbf{0.313}
\\
& 720 & 0.651 & \textbf{0.637} & 0.600 & \textbf{0.575} & 0.456 & \textbf{0.419} & 0.456  & \textbf{0.420} & 0.317 & \textbf{0.309} & 0.358 & \textbf{0.317}  &0.359 &\textbf{0.357} & 0.350 &\textbf{0.345} & 0.366 & \textbf{0.358} & 0.379 & \textbf{0.362}
\\
\hline
\multirow{4}{*}{\rotatebox[origin=c]{90}{\makecell[c]{Electricity}}}
& 96  & 0.289 & \textbf{0.281} & 0.381 & \textbf{0.364} & 0.202 & \textbf{0.197} & 0.317 & \textbf{0.307} & \textbf{0.140}  & 0.141 & 0.237 & \textbf{0.234} &\textbf{0.148} &0.149 &0.240 &\textbf{0.234} & \textbf{0.174} & 0.176 & 0.283 & \textbf{0.278}
\\
& 192 & 0.296 & \textbf{0.295} & 0.388 & \textbf{0.377} & 0.218 & \textbf{0.210}  & 0.331 & \textbf{0.316} & \textbf{0.154} & 0.155 & 0.250  & \textbf{0.246} &0.166 &\textbf{0.165} &0.258 &\textbf{0.250} & 0.172 & \textbf{0.170} & 0.275 & \textbf{0.271}
 \\
& 336 & \textbf{0.297} & 0.302 & 0.386 & \textbf{0.382} & \textbf{0.231} & \textbf{0.231} & 0.335 & \textbf{0.330}  & \textbf{0.169} & 0.170  & 0.268 & \textbf{0.262} &\textbf{0.178} &0.180 &0.271 &\textbf{0.266} & 0.184 & \textbf{0.181} & 0.284 & \textbf{0.278}
 \\
& 720 & \textbf{0.341} & 0.367 & \textbf{0.414} & 0.418 & 0.257 & \textbf{0.244} & 0.360  & \textbf{0.344} & \textbf{0.204} & \textbf{0.204} & 0.301 & \textbf{0.293} &\textbf{0.209} &\textbf{0.209} &0.299 &\textbf{0.291} & \textbf{0.218} & 0.220 & 0.304 & \textbf{0.300}
\\
\hline
\multicolumn{2}{c}{$1^{\mathrm{st}}$Count} & 6 & 26 & 4 & 28 & 7 & 27 & 3 & 29 & 5 & 28 & 1 & 31 & 7 & 26 & 1 & 31 & 5 & 28 & 0 & 32  \\
\hline
\end{tabular*}
\end{table*}

\begin{table}[t]
\centering
\caption{MSE comparison of DLinear at 720 steps under different noise levels. The dataset IDs follow the same order as in Table \ref{Table_result}.}
\label{Table_noise}
\fontsize{8}{8}\selectfont
\setlength{\tabcolsep}{3.5pt} 
\begin{tabular}{l|cccc|cccc}
\toprule
\multirow{2}{*}{ID} & \multicolumn{2}{c}{-3dB} & \multicolumn{2}{c}{0dB} & \multicolumn{2}{c}{3dB} & \multicolumn{2}{c}{10dB} \\
\cmidrule(lr){2-3} \cmidrule(lr){4-5} \cmidrule(lr){6-7} \cmidrule(lr){8-9}
 & MSE & Ours & MSE & Ours & MSE & Ours & MSE & Ours \\
\midrule
1 & 0.485 & \textbf{0.473} & 0.478 & \textbf{0.462} & 0.481 & \textbf{0.457} & 0.481 & \textbf{0.456} \\
2 & 0.693 & \textbf{0.579} & 0.657 & \textbf{0.565} & 0.649 & \textbf{0.558} & 0.613 & \textbf{0.551} \\
3 & 0.437 & \textbf{0.433} & 0.431 & \textbf{0.426} & 0.428 & \textbf{0.423} & 0.425 & \textbf{0.422} \\
4 & 0.471 & \textbf{0.391} & 0.464 & \textbf{0.384} & 0.459 & \textbf{0.380} & 0.455 & \textbf{0.375} \\
6 & 0.340 & \textbf{0.329} & 0.327 & \textbf{0.320} & 0.326 & \textbf{0.319} & 0.325 & \textbf{0.320} \\
\bottomrule
\end{tabular}
\vspace{-6pt} 
\end{table}

\begin{table}[t]
\centering
\fontsize{8}{8}\selectfont
\setlength{\tabcolsep}{3.5pt} 
\caption{Training time comparison (ms/iter) between MSE and RI-Loss at forecasting horizon 720 on the Weather dataset.}
\label{Table_efficiency}
\begin{tabular}{lccc}
\toprule
Method   & Dlinear & iTransformer & RAFT \\
\midrule
MSE      & 16.3 & 32.5         & 1386 \\
RI-Loss  & 30.5 & 40.2         & 1393 \\
\bottomrule
\end{tabular}
\vspace{-1.5em}
\end{table}

\begin{table}[ht]
\centering
\fontsize{8}{8}\selectfont
\setlength{\tabcolsep}{2pt}  
\caption{Comparison of different loss functions across five backbone models on the ETTh2 dataset.}
\label{Table_ablation}
\begin{tabular*}{\hsize}{@{\extracolsep{\fill}}c|c|cccc}
\toprule
\multirow{2}{*}{Models} & \multirow{2}{*}{Loss} & \multicolumn{4}{c}{ETTh2}  \\
\cmidrule(lr){3-6}
 & & 96 & 192 & 336 & 720  \\
\midrule
\multirow{4}{*}{Informer}
& RI-Loss        & \textbf{2.299}  & \textbf{3.949} & \textbf{3.412} & 3.094   \\
& MAE            & 2.444 & 3.995 & 3.713 & \textbf{3.083} \\
& MSE            & 3.340  & 5.684 & 4.418 & 3.249        \\
& Pearson$+$MSE  & 2.569 & 4.193 & 4.406 & 3.123       \\
\midrule
\multirow{4}{*}{Autoformer}
& RI-Loss        & \textbf{0.336}  & \textbf{0.417} & \textbf{0.444} & 0.445  \\
& MAE            & 0.337  & 0.420 & 0.449  & \textbf{0.443} \\
& MSE            & 0.375  & 0.451 & 0.477  & 0.476    \\
& Pearson$+$MSE  & 0.364  & 0.421 & 0.449  & 0.448      \\
\midrule
\multirow{4}{*}{DLinear}
& RI-Loss        & \textbf{0.279}  & \textbf{0.355} & \textbf{0.412} & \textbf{0.582}  \\
& MAE            & \textbf{0.279}  & 0.358  & 0.420  & 0.590     \\
& MSE            & 0.297  & 0.373  & 0.450  & 0.696     \\
& Pearson$+$MSE  & 0.305  & 0.404  & 0.494  & 0.714    \\
\midrule
\multirow{4}{*}{iTransformer}
& RI-Loss        & \textbf{0.296}  & \textbf{0.374} & \textbf{0.419} & \textbf{0.422}  \\
& MAE            & 0.298  & 0.377 & 0.421  & 0.427  \\
& MSE            & 0.301  & 0.380 & 0.424  & 0.430    \\
& Pearson$+$MSE  & 0.299  & 0.381 & 0.424  & 0.435    \\
\midrule
\multirow{4}{*}{RAFT}
& RI-Loss        & \textbf{0.293}  & \textbf{0.377} & \textbf{0.415} & \textbf{0.417}  \\
& MAE            & 0.294  & 0.379  & 0.420  & 0.422     \\
& MSE            & 0.296  & 0.384  & 0.426  & 0.434     \\
& Pearson$+$MSE  & 0.300  & 0.386  & 0.426  & 0.431    \\
\bottomrule
\end{tabular*}
\vspace{-8pt} 
\end{table}

\subsection{Experiments Settings}

\subsubsection{Datasets}
 We conduct extensive experiments on eight widely used real-world datasets, including ETT (ETTh1, ETTh2, ETTm1,
ETTm2), Exchange, ILI, Weather and Electricity. Please refer to Appendix C for data descriptions.

\subsubsection{Backbones}
We use five long-term time series forecasting models as our backbone models to comprehensively evaluate the effectiveness of RI-Loss. These include three Transformer-based models: Informer \cite{Informer}, Autoformer \cite{autoformer}, and iTransformer \cite{liu2024itransformer} and two MLP-based models: Decomposition-Linear (Dlinear) \cite{zeng2022} and Retrieval Augmented Time Series Forecasting (RAFT) \cite{han2025retrieval}.

\subsubsection{Evaluation Metrics}
Two widely used evaluation metrics in long-term time series forecasting, MSE and Mean Absolute Error (MAE), are employed to validate the proposed loss function. Lower values indicate better performance.

\subsubsection{Implementation Details}
We conduct our experiments using the official implementations of each backbone model from their respective repositories. To ensure a fair comparison, we adopt the original experimental setups and hyperparameter configurations when incorporating RI-Loss to enhance the performance of the backbone models. The hyperparameter $\lambda$ in RI-Loss is set to 10, and $\tau$ is set to 1. The kernel function is implemented as a Gaussian kernel with a bandwidth of 1. To mitigate the impact of randomness, all reported results are averaged over five independent runs.

\subsection{Main Results}

Table \ref{Table_result} presents a comparison of model performance using RI-Loss versus the traditional MSE Loss function.
Overall, RI-Loss yields better results. We analyzed the prediction error reduction across all datasets and found that RI-Loss leads to significant improvements in both the Informer and DLinear models. Specifically, it achieves an average reduction of 9.4\% in MSE and 6.9\% in MAE for Informer, and 5.2\% in MSE and 4.4\% in MAE for DLinear. Other models also show similar gains. Across 160 test cases, RI-Loss outperforms MSE loss in 133 instances, further demonstrating its robustness and broad applicability.

\begin{figure}[t]
\centering
\subfloat[MSE]{\includegraphics[width=0.24\textwidth,height=0.08\textwidth]{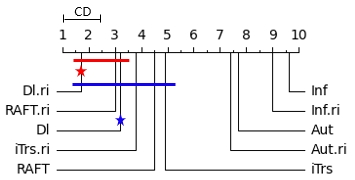}}
\subfloat[MAE]{\includegraphics[width=0.24\textwidth,height=0.08\textwidth]{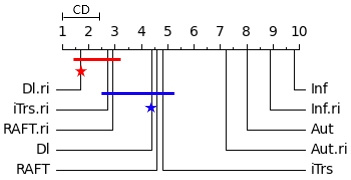}}
  \caption{The CD plots on two evaluation metrics (MSE and MAE) with a significance level at $ \alpha= 0.05$.}
  \label{fig_cd}
\vspace{-2em}
\end{figure}

To assess the statistical significance of performance differences between the two loss functions, we use the Friedman test. Critical Difference (CD) plots display the models’ average ranks, where lower ranks indicate better performance. In Figure \ref{fig_cd}, Inf, Aut, Dl, and iTrs represent Informer, Autoformer, DLinear, and iTransformer, respectively; suffix “.ri” denotes models trained with RI-Loss. Models A and B (highlighted by red and blue stars) correspond to the best-performing .ri model and the original model, serving as controls. Models not connected to A (red line) or B (blue line) differ significantly from these controls. Figure \ref{fig_cd} shows that all five models achieve improved rankings with RI-Loss. Further details on the Friedman test are in Appendix E.


\subsection{Noise Robustness}

To evaluate the robustness of our method under noisy conditions, we add zero-mean Gaussian noise to the input data at varying Signal-to-Noise Ratios (SNRs) and conduct experiments using DLinear as the backbone model. The results, summarized in Table~\ref{Table_noise}, show that our loss consistently outperforms the conventional MSE across all SNR levels. While MSE suffers significant performance degradation as noise increases, our method effectively mitigates the impact of noise, demonstrating better generalization and practical applicability for noisy time series forecasting.

\subsection{Running Cost}

To assess the computational overhead introduced by the RI loss, we adopt DLinear, iTransformer, and RAFT as backbone models and measure the per-iteration training time (ms/iter) on the Weather dataset under a forecasting horizon of 720, using both MSE loss and RI-loss.
As shown in Table \ref{Table_efficiency}, RI-Loss incurs a slight increase in computational time due to the higher overhead of HSIC compared to MSE. However, it consistently improves the prediction accuracy of backbone models. RI-Loss remains suitable for long-term time series forecasting on large-scale datasets. We include more results in the Appendix F.1.

\subsection{Ablation Study}

%
To further analyze the impact of different components in our loss function, we conducted four ablation experiments on the ETTh2 dataset: RI-Loss, MAE, MSE, and RI-Loss with HSIC replaced by the Pearson Correlation coefficient (PC). The results (Table \ref{Table_ablation}) show that replacing HSIC with PC degrades performance, as PC only captures linear relationships and cannot model nonlinear dependencies in complex time series. Meanwhile, RI-Loss outperforms both MSE and MAE, validating the effectiveness of our design. Further details are provided in Appendix F.2.

\subsection{Hyperparameter Sensitivity}\label{subsec4}

Our method introduces two key hyperparameters in the RI-Loss: the weight $\lambda$ and the temperature $\tau$.
We empirically set $\tau = 1$ to maintain a stable scaling of the HSIC term, and focus our analysis on $\lambda$. Using the iTransformer model, we evaluate $\lambda \in \{1, 5, 10, 15\}$. Figure \ref{Fig:lambda} shows that the prediction performance remains remarkably stable across this range, suggesting the model's robustness to $\lambda$ variations. Additional sensitivity analyses about $\tau $ are provided in Appendix F.3.

\begin{figure}[t]
\vspace{-0.5em}
\centering
\subfloat[ETTm1]{\includegraphics[width=0.15\textwidth,height=2cm]{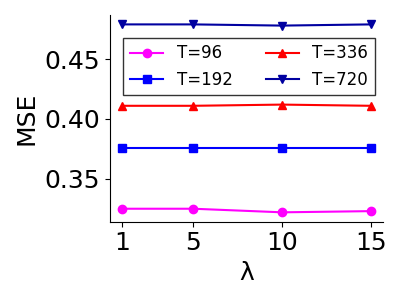}}
\subfloat[ETTh2]{\includegraphics[width=0.15\textwidth,height=2cm]{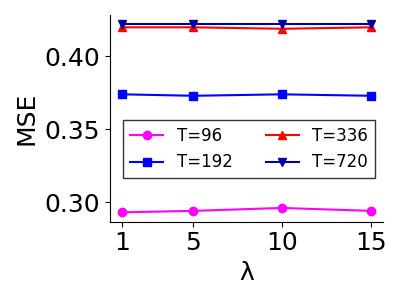}}
\subfloat[Weather]{\includegraphics[width=0.15\textwidth,height=2cm]{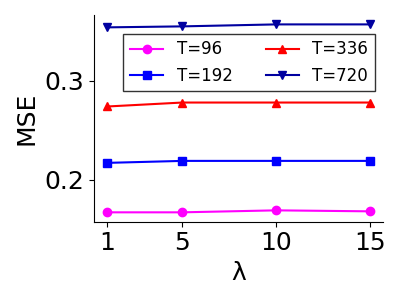}}

\caption{The impact of the hyperparameter $\lambda $ on iTransformer.}
\label{Fig:lambda}
\end{figure}

\subsection{Visualization of Prediction Results}
We use DLinear as the backbone model to visualize the prediction results. As shown in Figure \ref{figure_visual}, compared with the MSE loss, RI-Loss significantly improves the alignment with the ground truth. Predictions with RI-Loss align more closely with the ground truth’s statistical properties, yielding results with higher structural similarity. More visualizations are provided in Appendix F.4.

\begin{figure}[t]
\vspace{-8pt}
\centering
\subfloat[MSE]{\includegraphics[width=0.12\textwidth]{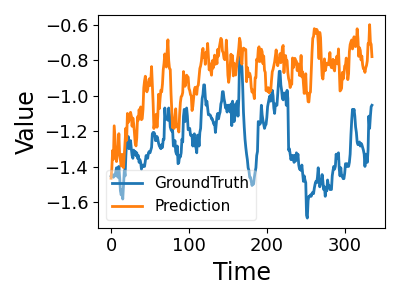}}
\subfloat[RI]{\includegraphics[width=0.12\textwidth]{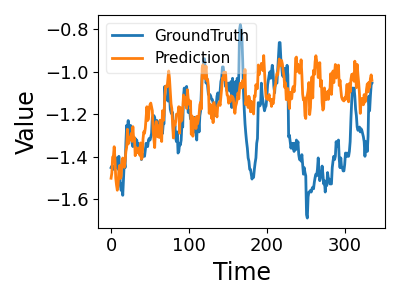}}
\subfloat[MSE]{\includegraphics[width=0.12\textwidth]{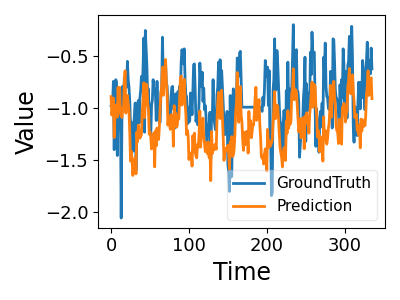}}
\subfloat[RI]{\includegraphics[width=0.12\textwidth]{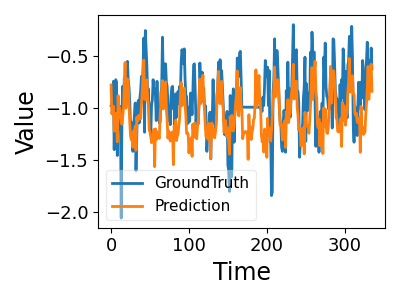}} \\
\caption{Forecasting visualization comparing RI-Loss and MSE loss as objective functions under the input-336-predict-336 settings. Blue lines are the ground truths and orange lines are the model predictions. Panels (a) and (b) correspond to ETTh1, with panels (c) and (d) representing ETTh2.}
\label{figure_visual}
\vspace{-16pt}
\end{figure}

\section{Conclusion}

Time series forecasting plays a pivotal role in various domains, yet existing models often rely on the MSE as a loss function, which suffers from critical limitations such as neglecting temporal dependencies and failing to handle inherent noise. To address these issues, this paper introduces RI-Loss, a novel loss function based on the Hilbert-Schmidt Independence Criterion (HSIC). RI-Loss leverages independence measures to improve the extraction of temporal patterns and mitigate noise interference. Additionally, theoretical learning bound for HSIC is proposed, providing theoretical support for its learnability. Experimental results on diverse benchmark datasets demonstrate the superior performance of RI-Loss, offering a significant advancement in time series forecasting.
Future directions include analyzing sophisticated noise correlations and residual dependencies to broaden the method and theoretical foundations.

%

\bibliography{aaai2026}


\clearpage
\newpage
\appendix

\maketitle

\section*{Appendix}
The appendix includes:
\begin{itemize}
  \item A Proofs.
  \item B Algorithm Framework.
  \item C Data Descriptions.
  \item D Backbone Models.
  \item E The Friedman Test.
  \item F More experimental results.
\end{itemize}

\section{A. Proofs}

This section includes:
\begin{itemize}
  \item A.1 Proof of Theorem \ref{crossterm}.
  \item A.2 Theoretical Analysis of HSIC Generalization: Proof Methodology Overview.
  \item A.3 Generalization Ability of HSIC.
  \item A.4 Bernstein-type Concentration Inequality.
  \item A.5 U-statistic and Related Work.
  \item A.6 Moment Bound for Degenerate U-Processes.
  \item A.7 Hoeffding's Decomposition of U-process.
  \item A.8 Generalization Bound for U-Processes.
\end{itemize}

\subsection{A.1 Proof of Theorem \ref{crossterm}}

\begin{theorem}[Cross-Term Expectation for Linear Projection]
Let $\bm{Y} = (Y_{t+1}, \dots, Y_{t+H})^\top \in \mathbb{R}^H$ be a noisy observation vector generated by:
$\bm{Y} = h(\bm{X}_t) + \bm{\epsilon},$
where $h(\bm{X}_t): \mathcal{X} \to \mathbb{R}^H$ is a deterministic mapping, and the noise vector $\bm{\epsilon} \in \mathbb{R}^H$ satisfies:
$\mathbb{E}[\bm{\epsilon}] = \bm{0},
\mathbb{E}[\bm{\epsilon}\bm{\epsilon}^\top] = \sigma^2\bm{I}_H,
\bm{\epsilon} \sim \mathcal{N}(\bm{0}, \sigma^2\bm{I}_H).$
Here, $\bm{I}_H$ denotes the $H\times H$ identity matrix and $\sigma^2 > 0$ is the noise variance.
For any linear estimator $\hat{\bm{Y}} = \bm{P}\bm{Y}$ with projection matrix $\bm{P} \in \mathbb{R}^{H \times H}$, the expected normalized cross-term evaluates to:
\begin{equation}
\mathbb{E}\left[\frac{1}{H}\langle \bm{Y} - \hat{\bm{Y}}, \bm{\epsilon} \rangle\right] = \frac{\sigma^2}{H}\mathrm{tr}(\bm{I}_H - \bm{P}),
\label{crossterm}
\end{equation}
where $\mathrm{tr}(\cdot)$ represents the trace of a matrix, defined as the sum of its diagonal elements.
\end{theorem}

\begin{proof}
We begin by expanding the expectation of the inner product term:
\begin{align*}
&\mathbb{E}_{\bm{\epsilon} }\left[\langle \bm{Y} - \hat{\bm{Y}}, \bm{\epsilon} \rangle\right] \\
&\quad = \mathbb{E}_{\bm{\epsilon} }\left[\mathrm{tr}\left((\bm{Y} - \hat{\bm{Y}})^\top \bm{\epsilon}\right)\right]
\quad \text{\footnotesize (Since $\langle \bm{a}, \bm{b} \rangle = \mathrm{tr}(\bm{a}^\top\bm{b})$)} \\
&\quad = \mathbb{E}_{\bm{\epsilon} }\left[\mathrm{tr}\left((\bm{I}_H-\bm{P})\bm{Y}\bm{\epsilon}^\top\right)\right]
\quad \text{\footnotesize (Substituting $\hat{\bm{Y}} = \bm{P}\bm{Y}$)} \\
&\quad = \mathbb{E}_{\bm{\epsilon} }\left[\mathrm{tr}\left((\bm{I}_H-\bm{P})(h(\bm{X}_t) + \bm{\epsilon})\bm{\epsilon}^\top\right)\right]  \\
&\quad = \mathrm{tr}\left((\bm{I}_H-\bm{P})h(\bm{X}_t)\mathbb{E}_{\bm{\epsilon} }[\bm{\epsilon}^\top]\right) \\
&\qquad + \mathrm{tr}\left((\bm{I}_H-\bm{P})\mathbb{E}_{\bm{\epsilon} }[\bm{\epsilon}\bm{\epsilon}^\top]\right)
\quad \text{\footnotesize (Linearity of expectation)} \\
&\quad = \mathrm{tr}\left((\bm{I}_H-\bm{P})h(\bm{X}_t)\cdot \bm{0}^\top\right) \\
&\qquad + \mathrm{tr}\left((\bm{I}_H-\bm{P})\sigma^2\bm{I}_H\right)
\quad \text{\footnotesize (Using noise properties)} \\
&\quad = 0 + \sigma^2\mathrm{tr}(\bm{I}_H-\bm{P})
\quad \text{\footnotesize (Simplifying both terms)} \\
&\quad = \sigma^2\mathrm{tr}(\bm{I}_H-\bm{P}),
\end{align*}
By dividing both sides by $H$, we obtain the final expression.
\end{proof}

\subsection{A.2 Theoretical Analysis of HSIC Generalization: Proof Methodology Overview}

We first give the definition of Rademacher complexity~\cite{Bartlett2001Rademacher}.
\begin{definition}[Empirical Rademacher Complexity]
Given a function class $\mathcal{G} \subseteq \mathbb{R}^\mathcal{X}$ and a fixed sample $S = \{x_1,...,x_n\} \subseteq \mathcal{X}$, the empirical Rademacher complexity is defined as:
\[
\mathcal{R}_\sigma(\mathcal{G}) := \mathbb{E}_\sigma\left[ \sup_{g \in \mathcal{G}} \frac{1}{n} \sum_{i=1}^n \sigma_i g(x_i) \right],
\]
where $\sigma_1,...,\sigma_n$ are independent Rademacher random variables ($\mathbb{P}(\sigma_i = \pm 1) = \frac{1}{2}$).
\end{definition}

Denote
$\mathbb{E}[\mathrm{HSIC}(R,S)]$ as the population HSIC between $R$ and $S$, and
$\mathrm{HSIC}(R,S)$ as the empirical HSIC estimator.
Among existing bounds for HSIC, the most relevant one to our work ~\cite{Greenfeld} is as follows.
For random variables $R \in \mathcal{R}$ and $S \in \mathcal{S}$ with joint distribution $\mathbb{P}_{RS}$, let $k: \mathcal{R} \times \mathcal{R} \to \mathbb{R}$ and $l: \mathcal{S} \times \mathcal{S} \to \mathbb{R}$ be characteristic kernels. Let $(R',S')$ be an independent and identically distributed (i.i.d.) copy of $(R,S)$.
Suppose that $\sup_{r,r'} |k(r,r')| \leq C_1\quad \text{and} \quad \sup_{s,s'} |l(s,s')| \leq C_2 \quad \text{(a.s.)},$
then for any $n \geq 2$ and $\delta \in (0,1)$, with probability at least $1-\delta$:
\begin{align}
\label{hsicbound0}
&\big|\mathrm{HSIC}(R,S)- \mathbb{E}[\mathrm{HSIC}(R,S)] \big| \\ \notag
&\leq 3C_1\left(4L \mathcal{R}_\sigma(\mathcal{H}) + \mathcal{O}\left(\sqrt{\frac{\ln(1/\delta)}{n}}\right)\right)   \\ \notag
&\quad \quad + 3C_2 C_1 \sqrt{\frac{\ln(2/\delta)}{2n}},
\end{align}
where $L$ is the Lipschitz constant of the kernel function $K$ with respect to $r$, $\mathcal{O}(\cdot)$ characterizes the asymptotic order, the variable $R$ represents the residual $R = Y - h(X)$, $h \in \mathcal{H}$ and
$\mathcal{R}_n(\mathcal{H})$ is the Rademacher complexity of $\mathcal{H}$.

Although this bound provides basic sample complexity and single sample Rademacher complexity guarantees, its ignorance of second-order property of kernel complexity motivates our deeper theoretical analysis presented below.

To verify the learnability of the HSIC-based loss, we propose a learning bound for HSIC based on self-bounded property and second-order Rademacher complexity~\cite{10.1007/BFb0097492}.
To the best of our knowledge, we are the first to establish a bound for HSIC based on the two-order complexity function.
The proof of generalization bound for HSIC consists of several parts:   
\begin{itemize}
\item \textbf{Degenerate U-processes Construction}
We employ Hoeffding's decomposition to represent the kernel function as a sum of centered kernel and a degenerate U-processes. This decomposition exploits the crucial property of degenerate U-statistics: controllable moment analysis based on two order Rademacher complexity.

\item \textbf{Bounded Differences Characterization}
Through careful analysis of the degenerate components, we show their bounded differences property. This enables the derivation of Bernstein-type concentration inequalities.

\item \textbf{Second-Order Complexity Bound Application}
We develop the deviation bound based on Bernstein-type concentration inequality and construct the expectation bound based on second-order Rademacher complexity measures.
The bound fundamentally improve the characterization of kernel complexity with respect to two samples.
\end{itemize}

The core of our theoretical framework lies (1) the construction of degenerate U-processes by Hoeffding's decomposition, (2) the control of U-statistic estimator by Bernstein-type bound, and (3) the control of moment via second-order Rademacher complexity measures.

For better readability and clarity, each component is presented in a separate section.
For readers interested primarily in the key findings, the conclusions are firstly presented in the following section titled \textit{Generalization Ability of HSIC}.

\subsection{A.3 Generalization Ability of HSIC}

Lemma \ref{estimate} states that the estimation error bound of HSIC is related to the error gap between the expected kernel and the empirical kernel.
\begin{lemma}~\cite{Greenfeld}
\label{estimate}
For random variables $R \in \mathcal{R}$ and $S \in \mathcal{S}$ with joint distribution $\mathbb{P}_{RS}$, let $k: \mathcal{R} \times \mathcal{R} \to \mathbb{R}$ and $l: \mathcal{S} \times \mathcal{S} \to \mathbb{R}$ be characteristic kernels. Let $(R',S')$ be an independent and identically distributed (i.i.d.) copy of $(R,S)$.
Let $c_2(k)=\sup_{r,r^{\prime}}k(r,r^{\prime})$, $c_2(l)=\sup_{s,s^{\prime}}l(s,s^{\prime})$.
Suppose that $ c_1(k) \leq k \leq c_2(k)$ and $ c_1(l) \leq l \leq c_2(l)$, then the following holds:
\begin{align}
\label{hsic}
& |\mathrm{HSIC}(S,R)-\mathbb{E}(\mathrm{HSIC}\left(\{(s_{i},r_{i})\}_{i=1}^{n}\right))| \\\notag
&\le 3c_{2}(k) \big |\mathbb{E}_{r,r^{\prime}}[k(r,r^{\prime})]-\frac{1}{\binom{n}{2}}\sum_{i_{1}\neq i_{2}}k(r_{i_{1}},r_{i_{2}})\big|  \\\notag
&+ 3c_{2}(l) \big|\mathbb{E}_{s,s^{\prime}}[l(s,s^{\prime})]-\frac{1}{\binom{n}{2}}\sum_{i_{1}\neq i_{2}  }l(s_{i_{1}},s_{i_{2}})\big|,
\end{align}
\end{lemma}
where $\binom{n}{2}$ represents the number of unique ways to select 2 distinct elements from a set of
$n$ elements without regard to order.

A Rademacher complexity based bound~\cite{Greenfeld} has been proposed to bound the estimation error of the kernel. However,
since the kernel function is defined over two samples, while Rademacher complexity is defined with respect to a single sample. The existing Rademacher complexity-based bound for HSIC~\cite{Greenfeld} requires improvement. Theorem \ref{kernelbound} in Appendix A.8.3 compensates for this deficiency by considering
different two order complexity functions.

Based on Theorem \ref{kernelbound}, we have:
\begin{theorem}
\label{hsicfinal}
Let $\sigma_1,...,\sigma_n$ are independent Rademacher random variables ($\mathbb{P}(\sigma_i = \pm 1) = \frac{1}{2}$), where $n$ is the number of samples.
For random variables $R \in \mathcal{R}$ and $S \in \mathcal{S}$ with joint distribution $\mathbb{P}_{RS}$, let $k: \mathcal{R} \times \mathcal{R} \to \mathbb{R}$ and $l: \mathcal{S} \times \mathcal{S} \to \mathbb{R}$ be characteristic kernels. Let $(R',S')$ be an independent and identically distributed (i.i.d.) copy of $(R,S)$.
Suppose that $ c_1(k) \leq k \leq c_2(k)$ and $ c_1(l) \leq l \leq c_2(l)$,
for $\delta>0$,
then the following holds:
\begin{align}
& \big|\mathrm{HSIC}(\{(r_i,s_i)\}_{i=1}^n)  - \mathbb{E}[\mathrm{HSIC}(R,S)]\big|  \\\notag
& \leq 3\sum_{j=k,l} c_2(j)\sum_{m=1}^3 \gamma_m(j).
\end{align}
Specifically, the above key functions for $f \in \{k,l\}$ are defined as:
\begin{align}
\label{zn2}
& \gamma_1(n,\delta;f) \\   \notag
&\quad=\frac{10(c_2(f)-c_1(f))}{n}  \ln \frac{2}{\delta}    \\   \notag
& \gamma_2(n,\delta,\mathcal{R}_{\sigma} ;f)    \\   \notag
& \quad=2(c_2(f)-c_1(f))  \sqrt{
   2\ln \frac{2}{\delta} \bigg(\frac{\mathcal{R}_{\sigma}  }{2n}  + \frac{1}{n^2} \ln \frac{2}{\delta} \bigg)  } \\   \notag
& \gamma_3( n,\delta,W_{\sigma,\sigma} ,W_{\sigma,\alpha} ,W_{\sigma} ,F;f)   \\   \notag
&\quad=4 \bigg( \ln \frac{2}{\delta} \bigg( \frac{C_0}{ n-1 }
 \big( \mathbb{E}W_{\sigma,\sigma}
 +\sqrt{2} \mathbb{E}W_{\sigma,\alpha}
 \\ \notag
 & +\frac{2 ( \mathbb{E}W_{\sigma} +F) }{n}
+ \frac{\sqrt{8} F  n^{1/2}}{n^2} +  \frac{4F}{n^2}  \big)
 +\frac{1 }{n^2}  \ln\frac{2}{\delta} \bigg) \bigg)^{\frac{1}{2}}.
\end{align}
In the above functions, the first-order and second-order sample Rademacher complexities are:
\begin{align}
\mathcal{R}_{\sigma} &=\mathbb{E}_\sigma\left[ \sup_{f \in \mathcal{G}} \frac{1}{n} \sum_{i=1}^n \sigma_i f_1(X_i) \right], \\
W_{\sigma,\sigma} &=\sup_{f\in\mathcal{F}}\frac{1}{n^2}\bigg|\sum_{i,j}\sigma_{i}\sigma_{j}f_2\left(X_{i},X_{j}\right)\bigg|, \quad \\
 W_{\sigma,\alpha} &=\sup_{f\in\mathcal{F}}\sup_{\alpha:\left\|\alpha\right\|_{2}\leq1}\frac{1}{n^2}\sum_{i,j}\sigma_{i}\alpha_{j}f_2\left(X_{i},X_{j}\right),\\
 W_{\sigma} &=\sup_{f\in\mathcal{F},k=1,\ldots,n}\frac{1}{n}\bigg|\sum_{i=1}^{n}\sigma_{i}f_2\left(X_{i},X_{k}\right)\bigg|,\\ \quad
 F&=sup_{f\in \mathcal{F} } \left \| f_2 \right \| _{\infty },
\end{align}
where $f_1$ and $f_2$ respectively represent the centered kernels:
\begin{align}
f_1(X_i) &= \mathbb{E}[f(X_i,X)|X_i] - \mathbb{E}[f(X_i,X_j)], \\ \nonumber
f_2(X_i,X_j) &= f(X_i,X_j) - f_1(X_i) - f_1(X_j) \\ \nonumber
&\quad \quad - \mathbb{E}[f(X_i,X_j)].
\end{align}
\end{theorem}

Next, we show the involved foundational concepts in the following four sections.
The deviation bound of U-Processes is presented in the section titled
\textit{Generalization Bound for U-Processes} and the final section titled \textit{Concentration Inequality of U-process} gives the proof of Theorem \ref{kernelbound}, the major result of this paper.

\subsection{A.4 Bernstein-type Concentration Inequality}\label{subsecB1}

Recent advances have established sharper Bernstein-type concentration inequalities for functions satisfying the $(a,b)$-self-bounding property \citep{boucheron2000sharp,McDiarmid2006ConcentrationFS,9765714,2000AboutTC}. These bounds typically outperform the Hoeffding-type bounds, as the self-bounding property provides more refined control through its dual constraints: local boundedness of individual variations and global dependence on the function's current value, rather than relying on uniform worst-case bounds.
\subsubsection{A.4.1 The property of \textbf{(a,b)-Self-bounding}}
\begin{definition}(\textbf{(a,b)-Self-bounding}~\cite{McDiarmid2006ConcentrationFS})
\label{def:self-bound}
Let $Z_n: \mathcal{X}^n \to \mathbb{R}$ be a measurable function. $Z_n$ has the (a,b)-self-bounding property if for any $\overline{X} = (X_1,\dots,X_n) \in \mathcal{X}^n$ consisting of independent random variables, there exist constants $a,b > 0$ and functions $Z_{n-1}^{(t)}: \mathcal{X}^{n-1} \to \mathbb{R}$ for $t = 1,\dots,n$ such that:
\begin{align}
& 0 \leq Z_n(\overline{X}) - Z_{n-1}^{(t)}(\overline{X}^{(t)}) \leq 1, \\
& \sum_{t=1}^n \left(Z_n(\overline{X}) - Z_{n-1}^{(t)}(\overline{X}^{(t)})\right) \leq a Z_n(\overline{X}) + b,
\end{align}
where $\overline{X}^{(t)} = (X_1,\dots,X_{t-1},X_{t+1},\dots,X_n)$ denotes the set with the $t$-th component removed.
\end{definition}

This definition describes a special class of functions $Z_N$ where, when any single input component is removed, the change in the function's value is both pointwise bounded (by at most 1) and globally constrained (with total variation not exceeding $aZ_N + b$). Such self-bounding property ensures controlled fluctuations of $Z_N$, making it particularly useful for proving concentration inequalities and exponential tail probability bounds in stochastic processes.
\subsubsection{A.4.2 Bernstein-type inequality}
\begin{theorem}\label{Bernstein}~\cite{McDiarmid2006ConcentrationFS})
If $Z$ is an $(a,b)$-self-bounding function, then for all $t > 0$, we have:
\begin{align}
\mathbb{P}\{Z \geq \mathbb{E}[Z] + t\} &\leq \exp\left(-\frac{t^2}{2(a\mathbb{E}[Z] + b + at)}\right), \\
\mathbb{P}\{Z \leq \mathbb{E}[Z] - t\} &\leq \exp\left(-\frac{t^2}{2(a\mathbb{E}[Z] + b + t/3)}\right).
\end{align}
\end{theorem}

The bounds in Theorem \ref{Bernstein} characterize the deviation probabilities of $Z$ around its mean $\mathbb{E}[Z]$, showing exponential decay that depends on three key factors: the self-bounding parameters $(a,b)$, the expectation $\mathbb{E}[Z]$, and the deviation size $t$. Based on Theorem \ref{Bernstein}, we have:

\begin{Corollary} (\textbf{Corollary of Theorem \ref{Bernstein}})  \label{elementary}
If $Z$ is an $(a,b)$-self-bounding function with $a > 1/3$, then for all $\delta > 0$, with probability at least $1-\delta$, we have:
\begin{align}
&|Z- \mathbb{E}[Z]| \leq a \ln (2/\delta) \\ \nonumber
&+ \sqrt{2\ln(2/\delta)(a\mathbb{E}[Z] + b) +a^2 (\ln(2/\delta))^2 }.
\label{deviationbound}
\end{align}
\end{Corollary}

\begin{proof}
Based on Theorem \ref{Bernstein}, we set the exponential probability bound of the first inequality as $\delta/2$. After solving for $t$,
we obtain:
$ 
t = a \ln (2/\delta) + \sqrt{2\ln(2/\delta)(a\mathbb{E}[Z] + b) +a^2(\ln(2/\delta))^2 }.
$ 
For the second inequality, we first amplify the probability upper bound as the same as the first one.
Then, combining the two inequalities, we obtain the upper bound of the deviation.
 \end{proof}

Corollary \ref{elementary} demonstrates that: for any $(a,b)$-self-bounding function $Z$ with $a>\frac{1}{3}$, its value concentrates around the expectation with high probability ($1-\delta$), where the deviation bound depends on $\mathbb{E}[Z]$, the self-bounding parameters $(a,b)$, and the confidence level $\delta$.

\subsection{A.5 U-statistic and Related Work}

Let $X_1, X_2, \ldots, X_n$ be independent and identically distributed (i.i.d.) random variables from a distribution $P$, and let $h:\mathcal{X}^m \to \mathbb{R}$ be a symmetric kernel function of order $m$ (i.e., $h$ is invariant under any permutation $\sigma$ of its arguments: $h(x_1,\ldots,x_m) = h(x_{\sigma(1)},\ldots,x_{\sigma(m)})$). Suppose $h$ is measurable and satisfies $\mathbb{E}_P[|h(X_1,\ldots,X_m)|] < \infty$, with
\[
\theta(P) = \mathbb{E}_P[h(X_1,\ldots,X_m)].
\]
\begin{definition}[U-statistic]
The U-statistic of degree $m$ associated with $h$ and $\theta(P)$ is defined as:
\[
U_n = \binom{n}{m}^{-1} \sum_{(i_1,\ldots,i_m) \in C_m^n} h(X_{i_1},\ldots,X_{i_m}),
\]
where $C_m^n$ denotes the set of all combinations of $m$ distinct indices $(i_1,\ldots,i_m)$ from $\{1,\ldots,n\}$.
\end{definition}

The corresponding U-statistic $U_n$ is the canonical unbiased estimator of $\theta(P)$, satisfying:
\[
\mathbb{E}[U_n] = \theta(P) \quad \text{for all } n \geq m.
\]
Moreover, $U_n$ achieves the minimum variance among all unbiased estimators of $\theta(P)$ when the kernel $h$ is square-integrable.

U-processes are significant in statistics because they are widely used in kernel method, graph data analysis, and metric learning scenarioss~\cite{gao2023,6795617,REJCHEL2015,Clmenon2014ASV}.
For example, if $h(X_i,X_j)$ measures similarity between samples, $U_n$ can characterize the heterogeneity of data distribution.
U-statistics have obtained significant theoretical and methodological advancements. Classical results on U-statistics, such as asymptotic properties~\cite{Gin1998}, have been extended to various modern settings. These include robust estimation techniques~\cite{JOLY}, sharp moment inequalities~\cite{00911}, and non-asymptotic tail probability analysis~\cite{1214,978-3}.

As a prototypical application of U-statistics theory, the existing bounds for HSIC in Eq. (\ref{hsicbound0}) has been derived based on the following Hoeffding-type tail bound for U-statistics~\cite{Gretton2005}:
if \( c_1 \leq h \leq c_2 \), then, for any \( t > 0 \), the following holds:
\begin{equation}
\mathbb{P}\left( U_n - \mathbb{E}[U_n] \geq t \right) \leq \exp\left( -\frac{2 \lfloor n/m \rfloor t^2}{(c_2 - c_1)^2} \right),
\label{hoedf}
\end{equation}
where \(n \) and \(m \) correspond to the number of observations and the degree of the U-statistic (order of the kernel function), respectively.

However, the concentration bound in Eq. (\ref{hoedf})
relies solely on the sample size and kernel order, yet additional factors affecting generalization must be included to provide a more complete characterization of convergence-related components.
For such cases, alternative bounds derived from Bernstein-type inequalities under sub-exponential conditions become necessary. For bounding the expectation component in the Bernstein-type bound, we invoke the explicit bounds developed in the next Section.

\subsection{A.6 Moment Bound for Degenerate U-Processes}

The following result establishes a moment bound for degenerate second-order U-processes.
Compared to Hoeffding or Talagrand inequalities, this theorem specifically targets degenerate U-processes, utilizing the symmetry and centering properties of kernel functions.
The explicit control of higher-order moments ($q \geq 2$) makes it suitable for non-asymptotic analysis in high-dimensional statistics.

\begin{theorem}~\cite{1214}\label{thm:U-process}
Let $\sigma_1, \ldots, \sigma_n$ be i.i.d. Rademacher random variables, and let $X, X_1, \ldots, X_n$ be i.i.d. random variables. Consider a class of kernels $\mathcal{F}$ satisfying:
\begin{enumerate}
    \item $\mathbb{E}[f(X,x)] = 0$ for all $x$ and $f \in \mathcal{F}$ (degeneracy condition)
    \item $f(x,x) = 0$ for all $x$ and $f \in \mathcal{F}$ (diagonal condition)
    \item $\sup_{f \in \mathcal{F}} \|f\|_\infty = F < \infty$  (Uniform Bound)
\end{enumerate}

Define the degenerate U-process of order 2:
\begin{equation}
Z = \sup_{f \in \mathcal{F}} \left| \sum_{i \neq j} f(X_i, X_j) \right|
\end{equation}

Then there exists a universal constant $C_0 > 0$ such that for all $n \geq 2$ and $q \geq 2$, the following moment bound holds:
\begin{align}
(\mathbb{E}[Z^q])^{1/q}
&\leq C_0 \Big( \mathbb{E}[Z_{\sigma,\sigma}] + q^{1/2}\mathbb{E}[Z_{\sigma,\alpha}] \notag \\
&\quad + q(\mathbb{E}[Z_\sigma] + Fn) + q^{3/2}Fn^{1/2} + q^2F \Big),
\label{eq:moment_bound}
\end{align}
where:
\begin{align}
Z_{\sigma,\sigma} &= \sup_{f \in \mathcal{F}} \left| \sum_{i,j} \sigma_i \sigma_j f(X_i, X_j) \right|, \\
Z_{\sigma,\alpha} &= \sup_{f \in \mathcal{F}} \sup_{\|\alpha\|_2 \leq 1} \sum_{i,j} \sigma_i \alpha_j f(X_i, X_j), \\
Z_\sigma &= \sup_{\substack{f \in \mathcal{F} \\ 1 \leq k \leq n}} \left| \sum_{i=1}^n \sigma_i f(X_i, X_k) \right|.
\end{align}
\end{theorem}

Eq. (\ref{eq:moment_bound}) controls the moments of the original U-process by decomposing it into double Rademacher processes~\cite{2087}, facilitating theoretical analysis.
By separating terms of different two order complexity (quadratic forms $Z_{\sigma,\sigma}$, linear constraints $Z_{\sigma,\alpha}$, single randomization $Z_\sigma$), Eq. (\ref{eq:moment_bound}) clarifies the contribution of each component.

The term $Z_{\sigma,\sigma}$ represents the supremum of a doubly Rademacher-weighted process over the function class $\mathcal{F}$.
It captures the maximum magnitude of the quadratic form $\sum_{i,j} \sigma_i \sigma_j f(X_i, X_j)$ for $f \in \mathcal{F}$,
where $\sigma_i$ are independent Rademacher random variables.
This quantity is particularly relevant for analyzing symmetric pairwise interaction models, such as those appearing in kernel methods and U-statistics.

The term $Z_{\sigma,\alpha}$ generalizes $Z_{\sigma,\sigma}$ by introducing an additional optimization over vectors $\alpha$ with $\|\alpha\|_2 \leq 1$.
It measures the supremum of the linearly constrained process $\sum_{i,j} \sigma_i \alpha_j f(X_i, X_j)$.
The additional flexibility makes it suitable for studying problems with adaptive weighting schemes,
including sparse learning and low-dimensional projection methods.

The term $Z_\sigma$ evaluates the worst-case behavior of the Rademacher process when one argument is held fixed.
Specifically, it computes the maximum absolute value of $\sum_{i=1}^n \sigma_i f(X_i, X_k)$ over all $f \in \mathcal{F}$ and all possible anchor points $X_k$.
This formulation naturally arises in problems involving center-based learning,
such as clustering algorithms and metric learning scenarios where stability around reference points is crucial.

\subsection{A.7 Hoeffding's Decomposition of U-process}\label{subsecB3}
Let $\{X_1, X_2, \ldots, X_n\}$ be a sample of independent and identically distributed random variables from some probability distribution. A two-order U-process $Z_n$ with respect to a symmetric kernel $f:\mathcal{X} \times \mathcal{X} \to \mathbb{R}$ (where $f(x,z) = f(z,x)$) is defined as:
\begin{equation}
Z_n(f) = \frac{1}{n(n-1)} \sum_{i=1}^{n} \sum_{j\neq i}^{n} f(X_i,X_j),
\label{znzn}
\end{equation}
where $n(n-1)$ counts the unique pairs in the sample. The convergence analysis of $Z_n$ typically employs Hoeffding's decomposition\cite{Hoeffding1992}:
\begin{align}\label{decomposition}
Z_n(f) &= \mathbb{E}[Z_n(f)] + 2T_n(f_1) + W_n(f_2),
\end{align}
where $T_n$ is an order-one term and $W_n$ is a degenerate U-statistic satisfying $\mathbb{E}[W_n(x,X)] = 0$ for any $x$. The components are:
\begin{align}\label{Tn}
T_n &= \frac{1}{n} \sum_{i=1}^n f_1(X_i), \notag \\
W_n &= \frac{1}{n(n-1)} \sum_{i=1}^n \sum_{j \neq i} f_2(X_i,X_j).
\end{align}

The centered kernels are defined through:
\begin{align}\label{f1}
f_1(X_i) = \mathbb{E}[f(X_i,X)|X_i] - \mathbb{E}[f(X_i,X_j)],
\end{align}
\begin{align}\label{f2}
f_2(X_i,X_j) &= f(X_i,X_j) - f_1(X_i) - f_1(X_j) \\ \nonumber
&\quad \quad - \mathbb{E}[f(X_i,X_j)].
\end{align}

This decomposition yields the key properties $\mathbb{E}[T_n] = 0$ and $\mathbb{E}[W_n(x,X)] = 0$.

Through the Hoeffding decomposition,
we can bound the deviation probability of the U-statistic $Z_n$ by considering the individual components:
\begin{align}\label{tt}
\mathbb{P}\big(|Z_n(f)-\mathbb{E}[Z_n(f)]| \geq t_1+t_2\big) &\leq \mathbb{P}\big(|2T_n(f_1)| \geq t_1\big) \\ \notag
&\quad + \mathbb{P}\big(|W_n(f_2)| \geq t_2\big),
\end{align}
where $t_1, t_2 > 0$ are any positive values. Then,
The fluctuations of $T_n$ and $W_n$ can be controlled via Bernstein-type inequality, which provides exponential tail bounds for their deviations from their means. The second-order statistic $W_n$ admits an upper expectation bound characterized by Theorem~\ref{thm:U-process}, which governs the behavior of U-processes.

\subsection{A.8 Generalization Bound for U-Processes}
The objective of this section is to establish upper bounds for the deviation between empirical U-processes and their expectations. Specifically, we aim to control
$|Z_n(f) - \mathbb{E}Z_n(f)|,$
where $Z_n(f)$ represents the empirical U-process defined as Eq. (\ref{znzn}) and $\mathbb{E}Z_n(f)$ its population counterpart.

We introduce the sum of the conditional and total centered kernel functions:
\begin{align}
T'_n(f_1) &= \frac{1}{ c_2 - c_1}\sum_{i=1}^n f_1(X_i), \\
W'_n(f_2) &= C\sum_{i=1}^{n}\sum_{j\neq i}^{n} f_2(X_i,X_j), \\
C&=\frac{1}{4(n-1)(c_2 - c_1)}.
\end{align}

The concentration of the U-process can be obtained through $T'_n$ and $W'_n$ via:
\begin{align}
\label{tw}
\mathbb{P}(|Z_n(f) & -\mathbb{E}Z_n(f)| \geq t_1+t_2)   \\ \nonumber
&\leq \mathbb{P}(|2T_n(f_1)| \geq t_1) + \mathbb{P}(|W_n(f_2)| \geq t_2) \\ \nonumber
&\leq \mathbb{P}\left(\frac{2(c_2-c_1)}{n}|T'_n(f_1)| \geq t_1\right)  \\ \nonumber
&\quad + \mathbb{P}\left(\frac{4(c_2-c_1)}{n}|W'_n(f_2)| \geq t_2\right) \\\nonumber
&\leq \mathbb{P}\left(\frac{2(c_2-c_1)}{n}\sup_{f\in\mathcal{F}}|T'_n(f_1)| \geq t_1\right)   \\\nonumber
&\quad + \mathbb{P}\left(\frac{4(c_2-c_1)}{n}\sup_{f\in\mathcal{F}}|W'_n(f_2)| \geq t_2\right),
\end{align}
where $c_1$ and $c_2$ are the lower and upper bounds of $f$ respectively, and $t_1$, $t_2$ are arbitrary positive numbers.

Followed closely by proving that both $T'_n$ and $W'_n$ are $(a,b)$ self-bounding, and based on
Corollary \ref{elementary} and Theorem \ref{thm:U-process}, we obtain the concentration inequalities of $T'_n$ and $W'_n$, shown as
Lemma \ref{Tn1} and Lemma \ref{wn}, respectively. Then through some straightforward calculations, with probability at least $1-2\delta$, we can obtain the final result in Theorem \ref{estimate2}.

\subsubsection{A.8.1 High-Probability Bound for $\sup_{f\in\mathcal{F}} |T'_n|$}


\begin{Property}
\label{c1}
The  linear term  $\sup_{f\in\mathcal{F}} |T'_n|$ satisfies the $(a,b)$ self-bounding property with $a=1, b=0$.
\end{Property}

\begin{lemma}
\label{Tn1}
For all $\delta>0$, with probability at least $1-\delta$,
we have,
$ 
\sup_{f\in\mathcal{F}} |T'_n(f_1)| \leq   \ln (2/\delta)
+ \sqrt{2\ln(2/\delta) \mathbb{E}[\sup_{f\in\mathcal{F}} |T'_n(f_1)|]   + (\ln(2/\delta))^2 }.
$ 
Furthermore, according to the symmetrization inequality~\cite{9765714},
we have
$ 
 \mathbb{E}[\sup_{f\in\mathcal{F}} |T'_n(f_1)|]\leq \frac{n}{2} \mathcal{R}_{\sigma},
$ 
where $\mathcal{R}_{\sigma}$ is the Rademacher complexity.
\end{lemma}

\subsubsection{Proof of Property \ref{c1}.}

According to Definition \ref{def:self-bound}, we need to bound the deviation:
\(\sup_{f\in\mathcal{F}}  \big|T'_n (\overline{X}) \big| - \sup_{f\in\mathcal{F}}  \big|T'_{n-1} (\overline{X}^{(t)} )\big|\) and its summation over $t$.

Suppose that $\sup_{f\in\mathcal{F}}  |T'_n (\overline{X}) |  $ maximizes at  $\hat{f}_1$ and $\sup_{f\in\mathcal{F}} |T'_n (\overline{X}^{(t)}) |  $ maximizes at $\hat{f}_{1,t}$, we have,
\begin{align}
\sup_{f\in\mathcal{F}}  |T'_n (\overline{X}) |
& =
\frac{1}{ c_2 - c_1}|\sum_{i =1 }^{n}   \hat{f}_1(X_i) |, \\ \label{wnd}
\sup_{f\in\mathcal{F}}  |T'_{n-1} (\overline{X}^{(t)}) |
& =
\frac{1}{ c_2 - c_1}|\sum_{ \substack{i =1 \\  i\neq t  } }^{n}   \hat{f}_{1,t} (X_i)|.
\end{align}

Firstly, we prove the lower bound $\sup_{f\in\mathcal{F}}  \big|T'_n (\overline{X}) \big| - \sup_{f\in\mathcal{F}}  \big|T'_{n-1} (\overline{X}^{(t)} )\big|\geq0$.
Set the value of $\hat{f}_{1,t}$ at $X_t$ as 0, that is,
$\hat{f}_{1,t} (X_t)=0 $. Then because $\hat{f}_1$ is the supremum solution of $\big|T'_n (\overline{X}) \big|$, we can obtain the inequality.

Secondly,
to prove the upper bound $\sup_{f\in\mathcal{F}}  |T'_n (\overline{X}) | - \sup_{f\in\mathcal{F}}  |T'_{n-1} (\overline{X}^{(t)}) |\leq1$, we separate the terms involving $X_t$ from $\sup_{f\in\mathcal{F}}  |T'_n (\overline{X}) |$ as follows:
\begin{align}
& (c_2 - c_1) (\sup_{f\in\mathcal{F}}  |T'_n (\overline{X}) |- \sup_{f\in\mathcal{F}}  |T'_{n-1} (\overline{X}^{(t)}))\\
& =\big|\sum_{i =1 ,i \neq t}^{n}   \hat{f}_1(X_i)+ \hat{f}_1(X_t) \big|-\sup_{f\in\mathcal{F}}  |T'_{n-1} (\overline{X}^{(t)}) | \\
&\leq  |\hat{f}_1(X_t)|+ \big|\sum_{i =1 ,i \neq t}^{n}   \hat{f}_1(X_i) \big|-\sup_{f\in\mathcal{F}}  |T'_{n-1} (\overline{X}^{(t)}) | \\
&\leq  |\hat{f}_1(X_t)| \leq (c_2 - c_1),
\label{gapt}
\end{align}
where the first inequality is by the triangle inequality, the second one is by the supremum of $\big|T'_{n-1} (\overline{X}^{(t)} ) \big |$, and the last one is due to $| \hat{f}_{1,t} (X_t)|\leq (c_2 - c_1)$ when $f$ is bounded $ c_1 \leq f \leq c_2$.

Thirdly, we prove the upper bound of sum $\sum _{t=1}^n \big( \sup_{f\in\mathcal{F}}  |T'_n (\overline{X}) | - \sup_{f\in\mathcal{F}}  |T'_{n-1} (\overline{X}^{(t)}) | \big)
 \leq a\sup_{f\in\mathcal{F}} |T'_n (\overline{X}) |+b $, where $a=1 ,b=0$. This is because by
Eq. (\ref{gapt}), we have,
$\sum _{t=1}^n  |\hat{f}_1(X_t)|   \leq      \sup_{f\in\mathcal{F}}   \sum _{t=1}^n   | {f}_1(X_t)  |$.

\subsubsection{A.8.2 High-Probability Bound for $\sup_{f\in\mathcal{F}} |W'_n|$}

Firstly, we prove that $\sup_{f\in\mathcal{F}} |W'_n|$ is $(a,b)$ self-bounding.
\begin{Property}
\label{c2}
The degenerate U-statistic $\sup_{f\in\mathcal{F}} |W'_n|$ satisfies the $(a,b)$ self-bounding property with $a=2, b=0$.
\end{Property}

Then, according to Corollary \ref{elementary}, we have:

\begin{lemma}
\label{wn}
For all $\delta>0$, with probability at least $1-\delta$,
we have,
$ 
  \sup_{f\in\mathcal{F}} |W'_n(f_2)|
\leq 2 \ln (2/\delta)
 + \sqrt{2\ln(2/\delta)(2\mathbb{E}[\sup_{f\in\mathcal{F}} |W'_n(f_2)|]  ) +4(\ln(2/\delta))^2 }.
$ 

Furthermore, according to Theorem \ref{thm:U-process}, we have,
\begin{align}
&\mathbb{E}\big [\sup_{f\in\mathcal{F}} |W'_n(f_2)| \big] \notag \\
&\leq \bigg( \mathbb{E}\big [ \big( \sup_{f\in\mathcal{F}} |W'_n| \big)^2 \big] \bigg)^{1/2} \notag \\
&\leq CC_0 \Bigg( n^2\mathbb{E}W_{\sigma,\sigma}  + 2^{1/2}n^2\mathbb{E}W_{\sigma,\alpha}  \notag \\
&\quad + 2 \Big(n \mathbb{E}W_{\sigma}  + Fn\Big) + 2^{3/2}Fn^{1/2} + 2^{2}F \Bigg),
\end{align}
where,
\begin{align}
W_{\sigma,\sigma}  &= \sup_{f\in\mathcal{F}}\frac{1}{n^2}\bigg|\sum_{i,j}\sigma_{i}\sigma_{j}f_2\left(X_{i},X_{j}\right)\bigg|, \\
W_{\sigma,\alpha} &= \sup_{f\in\mathcal{F}}\sup_{\alpha:\left\|\alpha\right\|_{2}\leq1}\frac{1}{n^2}\sum_{i,j}\sigma_{i}\alpha_{j}f_2\left(X_{i},X_{j}\right), \\
W_{\sigma} &= \sup_{f\in\mathcal{F},k=1,\ldots,n}\frac{1}{n}\bigg|\sum_{i=1}^{n}\sigma_{i}f_2\left(X_{i},X_{k}\right)\bigg|, \\
F &= \sup_{f\in \mathcal{F}} \left\| f_2 \right\|_{\infty}.
\end{align}
\end{lemma}

\subsubsection{Proof of Property \ref{c2}.}

According to Definition \ref{def:self-bound}, we need to bound the deviation:
\(\sup_{f\in\mathcal{F}}  \big|W'_n (\overline{X}) \big| - \sup_{f\in\mathcal{F}}  \big|W'_{n-1} (\overline{X}^{(t)} )\big|\) and its summation over $t$.

Suppose that $\sup_{f\in\mathcal{F}}  |W'_n (\overline{X}) |  $ maximizes at  $\hat{f}_2$ and $\sup_{f\in\mathcal{F}} |W'_n (\overline{X}^{(t)}) |  $ maximizes at $\hat{f}_{2,t}$, we have,
\begin{align}
\sup_{f\in\mathcal{F}}  |W'_n (\overline{X}) |
& =
C|\sum_{i =1 }^{n} \sum_{j\neq i }^{n}   \hat{f}_2(X_i,X_j) |, \\ \label{wnd}
\sup_{f\in\mathcal{F}}  |W'_{n-1} (\overline{X}^{(t)}) |
& =
C|\sum_{ \substack{i =1 \\  i\neq t  } }^{n} \sum_{  \substack{ j\neq i \\ j\neq t }  }^{n}   \hat{f}_{2,t} (X_i,X_j)|,\\
& C=\frac{1}{4(n-1)(c_2 - c_1)}.
\end{align}

Firstly, we prove the lower bound $\sup_{f\in\mathcal{F}}  \big|W'_n (\overline{X}) \big| - \sup_{f\in\mathcal{F}}  \big|W'_{n-1} (\overline{X}^{(t)} )\big|\geq0$.
Set the values of $\hat{f}_{2,t}$ at $X_t$ as 0, that is,
$\hat{f}_{2,t} (X_t,X_j)=0 $ and  $\hat{f}_{2,t} (X_i,X_t)=0$. Then because $\hat{f}_2$ is the supremum solution of $\big|W'_n (\overline{X}) \big|$, we can obtain the inequality.

Secondly,
to prove the upper bound $\sup_{f\in\mathcal{F}}  |W'_n (\overline{X}) | - \sup_{f\in\mathcal{F}}  |W'_{n-1} (\overline{X}^{(t)}) |\leq1$, we separate the terms involving $X_t$ from $\sup_{f\in\mathcal{F}}  |W'_n (\overline{X}) |$ as follows:
\begin{align}
&\sup_{f\in\mathcal{F}} \big|W'_n (\overline{X}) \big| - \sup_{f\in\mathcal{F}} \big|W'_{n-1} (\overline{X}^{(t)}) \big|  \\
&\leq C \bigg | \sum_{j\neq t }^{n}  \hat{f}_{2,t} (X_t,X_j)\bigg |+ C \bigg | \sum_{i\neq t }^{n}  \hat{f}_{2,t} (X_i,X_t)\bigg | \\
&+C \bigg|\sum_{ \substack{i =1 \\  i\neq t  } }^{n}  \sum_{  \substack{ j\neq i \\ j\neq t }}^{n} \hat{f}_{2,t} (X_i,X_j) \bigg|-
\sup_{f\in\mathcal{F}} \big|W'_{n-1} (\overline{X}^{(t)} ) \big |\\
&\leq C \bigg | \sum_{j\neq t }^{n}  \hat{f}_{2,t} (X_t,X_j)\bigg |+ C \bigg | \sum_{i\neq t }^{n}  \hat{f}_{2,t} (X_i,X_t)\bigg | \\
& \leq 2C \bigg | \sum_{j\neq t }^{n}  \hat{f}_{2,t} (X_t,X_j)\bigg |  \leq 1,
\label{gap}
\end{align}
where the first inequality is by the triangle inequality, the second one is by the supremum of $\big|W'_{n-1} (\overline{X}^{(t)} ) \big |$, and the last one is due to $| \hat{f}_{2,t} (X_t,X_j)|\leq2(c_2 - c_1)$ when $f$ is bounded $ c_1 \leq f \leq c_2$.

Thirdly, we prove the upper bound of sum $\sum _{t=1}^n \big( \sup_{f\in\mathcal{F}}  |W'_n (\overline{X}) | - \sup_{f\in\mathcal{F}}  |W'_{n-1} (\overline{X}^{(t)}) | \big)
 \leq a\sup_{f\in\mathcal{F}} |W'_n (\overline{X}) |+b $, where $a=2 ,b=0$. This is because by
Eq. (\ref{gap}), we have,
\begin{align}
&\sum _{t=1}^n  2C | \sum_{j\neq t }^{n}  \hat{f}_{2,t} (X_t,X_j) | \\
& \leq 2C    \sup_{f\in\mathcal{F}}   \sum _{t=1}^n   | \sum_{j\neq t }^{n}  f_{2,t} (X_t,X_j) |\\
&\leq 2C \sup_{f\in\mathcal{F}}    |\sum _{t=1}^n  \sum_{j\neq t }^{n}   f_{2,t} (X_t,X_j) |,
\end{align}
where the last equality is satisfied because that for
the symmetric $\mathcal{F}$ with $f\in \mathcal{F}$ and $-f \in \mathcal{F}$, we have
$  \sup_{f\in\mathcal{F}} |f(u)|+  |f(v)|  =\sup_{f\in\mathcal{F}}|f(u)+f(v)|  $.
Hence, by the definition of $\sup_{f\in\mathcal{F}} |W'_n (\overline{X}) |$, we obtain the final result.

\subsubsection{A.8.3 Concentration Inequality of U-process}\label{subsecB4}

Combining Eq. (\ref{tw}), Lemma \ref{Tn1}, and Lemma \ref{wn}, we have:
\begin{theorem} A two-order U-process $Z_n$ with respect to a symmetric kernel $f:\mathcal{X} \times \mathcal{X} \to \mathbb{R}$ (where $f(x,z) = f(z,x)$) is defined as:
\begin{equation}
Z_n(f) = \frac{1}{n(n-1)} \sum_{i=1}^{n} \sum_{j\neq i}^{n} f(X_i,X_j),
\label{U}
\end{equation}
where $n(n-1)$ counts the unique pairs in the sample. Suppose that $ c_1 \leq f \leq c_2$,
for $\delta>0$,
with probability at least $1-2\delta$, we have:
\label{estimate2}
\begin{align}
\label{zn}
&|Z_n(f) -\mathbb{E}Z_n(f)|
\leq   \frac{10(c_2-c_1)}{n}  \ln \frac{2}{\delta}    \\   \notag
&  +    2(c_2-c_1)   \sqrt{
   2\ln \frac{2}{\delta} \bigg(\frac{\mathcal{R}_{\sigma}  }{2n}  + \frac{1}{n^2} \ln \frac{2}{\delta} \bigg)  } \\   \notag
&+4 \bigg( \ln \frac{2}{\delta} \bigg( \frac{C_0}{ n-1 }
 \big( \mathbb{E}W_{\sigma,\sigma}
 +\sqrt{2} \mathbb{E}W_{\sigma,\alpha}
 \\ \notag
 & +\frac{2 ( \mathbb{E}W_{\sigma}+F) }{n}
+ \frac{\sqrt{8} F  n^{1/2}}{n^2} +  \frac{4F}{n^2}  \big)
 +\frac{1 }{n^2}  \ln\frac{2}{\delta} \bigg) \bigg)^{\frac{1}{2}}.
\end{align}

where,
\begin{align}
&\mathcal{R}_{\sigma}=\mathbb{E}_\sigma\left[ \sup_{f \in \mathcal{G}} \frac{1}{n} \sum_{i=1}^n \sigma_i f_1(X_i) \right], \\
&W_{\sigma,\sigma}=\sup_{f\in\mathcal{F}}\frac{1}{n^2}\bigg|\sum_{i,j}\sigma_{i}\sigma_{j}f_2\left(X_{i},X_{j}\right)\bigg|, \quad \\
&W_{\sigma,\alpha}=\sup_{f\in\mathcal{F}}\sup_{\alpha:\left\|\alpha\right\|_{2}\leq1}\frac{1}{n^2}\sum_{i,j}\sigma_{i}\alpha_{j}f_2\left(X_{i},X_{j}\right),\\
&W_{\sigma}=\sup_{f\in\mathcal{F},k=1,\ldots,n}\frac{1}{n}\bigg|\sum_{i=1}^{n}\sigma_{i}f_2\left(X_{i},X_{k}\right)\bigg|,\\ \quad
&F=sup_{f\in \mathcal{F} } \left \| f_2 \right \| _{\infty }.
\end{align}
\label{kernelbound}
\end{theorem}
\begin{proof}
The theorem follows by combining the aforementioned results. However, to enhance comprehension, we provide partial derivations below:
according to Eq. (\ref{tw}), we have:
\begin{align}
\label{yu}
&\mathbb{P}(|Z_n(f)   -\mathbb{E}Z_n(f)| \geq t_1+t_2) \\
&\leq \mathbb{P}\left(\frac{2(c_2-c_1)}{n} \sup_{f\in\mathcal{F}}|T'_n(f_1)| \geq t_1  \right)   \\
&\quad + \mathbb{P}\left(\frac{4(c_2-c_1)}{n} \sup_{f\in\mathcal{F}}|W'_n(f_2)| \geq t_2 \right).
\end{align}

According to Lemma \ref{Tn1},
for all $\delta>0$, with probability at least $1-\delta$,
we have,
$ 
\sup_{f\in\mathcal{F}} |T'_n(f_1)| \leq   \ln (2/\delta)
+ \sqrt{2\ln(2/\delta)  \frac{n}{2} R_{\sigma}    + (\ln(2/\delta))^2 }.$
According to Lemma \ref{wn}, for all $\delta>0$, with probability at least $1-\delta$,
we have,
\begin{align}
&\sup_{f\in\mathcal{F}} |W'_n(f_2)| \\
&\leq 2 \ln (2/\delta) \notag \\
& + \Bigg\{ 2\ln(2/\delta) \cdot \bigg( 2CC_0 \Big( n^2\mathbb{E}W_{\sigma,\sigma}
                                          + 2^{1/2}n^2\mathbb{E}W_{\sigma,\alpha} ) \notag \\
& + 2 \big(n \mathbb{E}W_{\sigma}  + Fn\big)
                                          + 2^{3/2}Fn^{1/2}
                                          + 2^{2}F \Big) \bigg) \notag \\
&  + 4(\ln(2/\delta))^2 \Bigg\}^{1/2}.
\label{deviationbound2}
\end{align}
Combining the above result, we obtain the final inequality.
\end{proof}

\begin{table*}[t]
    \vspace{-0.8em}
    \centering
    \renewcommand{\arraystretch}{1}  
    \caption{Statistics of the eight datasets used in experiments. “Frequency” denotes the sampling interval of time points.}
    \label{Table_data}
    \begin{tabular}{lcccccccc}
        \hline
        Dataset       & ETTh1 & ETTh2 & ETTm1 & ETTm2 & Exchange & ILI & Weather & Electricity \\
        \hline
        Features      & 7     & 7     & 7     & 7     & 8        & 7   & 21      & 321         \\
        Timesteps     & 17420 & 17420 & 69680 & 69680 & 7588     & 966 & 52696   & 26304       \\
        Frequency     & 1 hour     & 1 hour     & 15 min     & 15 min     & 1 day       & 1 week   & 10 min     & 1 hour \\
        \hline
    \end{tabular}
    \vspace{-0.8em}
\end{table*}

\section{B. Algorithm Framework}
We present the complete integration procedure of RI-Loss with the backbone model in Algorithm~\ref{alg: algorithm_hsic}.

\begin{algorithm}[H]
\caption{RI-Loss Based Time Series Model with HSIC Regularization}
\label{alg: algorithm_hsic}
\begin{algorithmic}[1]
\small
\REQUIRE Time series dataset $\mathcal{D} = \{(\boldsymbol{X}_{1:T}^{(i)}, \boldsymbol{Y}_{T+1:T+H}^{(i)})\}_{i=1}^N$, batch size $B$
\REQUIRE Hyperparameters: $\lambda$, $\tau$, learning rate $\eta$, epochs $E$
\ENSURE Trained model $G_{\theta}$

\STATE Initialize $\theta$ \COMMENT{Parameter initialization}
\STATE $K \gets \lceil N/B \rceil$ \COMMENT{Compute total batches}

\FOR{epoch $= 1$ \TO $E$}
    \STATE $\{\mathcal{B}_k\}_{k=1}^K \gets \text{ShuffleSplit}(\mathcal{D})$
    \FOR{batch $\mathcal{B}_k \in \{\mathcal{B}_{1:K}\}$}
        \STATE $(\boldsymbol{X}_{1:T}, \boldsymbol{Y}_{T+1:T+H}) \gets \mathcal{B}_k$
        \STATE $\hat{\boldsymbol{Y}}_{T+1:T+H} \gets G_{\theta}(\boldsymbol{X}_{1:T})$

        \STATE Generate $\epsilon \sim \mathcal{U}(-1,1)^{B \times H \times d}$ \COMMENT{Uniform noise matching residual dimensions}
        \STATE $\mathcal{L}_{\text{MSE}} \gets \frac{1}{BHd}\sum_{i=1}^B\sum_{t=T+1}^{T+H}\|\boldsymbol{Y}_t^{(i)}-\hat{\boldsymbol{Y}}_t^{(i)}\|_2^2$ \COMMENT{Normalized by dimensions}
        \STATE $\mathcal{L}_{\text{HSIC}} \gets \exp(-\tau\cdot\text{HSIC}(\boldsymbol{Y}_{T+1:T+H}-\hat{\boldsymbol{Y}}_{T+1:T+H}, \epsilon))$
        \STATE $\mathcal{L} \gets \mathcal{L}_{\text{MSE}} + \lambda\mathcal{L}_{\text{HSIC}}$

        \STATE $\theta \gets \theta - \eta\nabla_{\theta}\mathcal{L}$ \COMMENT{Gradient update}
    \ENDFOR
\ENDFOR
\RETURN $G_{\theta}$
\end{algorithmic}
\end{algorithm}

\section{C. Data Descriptions} \label{ref_data}
We use eight widely-used datasets in the main paper. The details are listed in the following.
\begin{itemize}
  \item ETT (Electricity Transformer Temperature) \cite{Informer}: This dataset includes 7 variables related to electricity transformers, covering the period from July 2016 to July 2018. It is divided into four subsets: ETTh1 and ETTh2, which are sampled hourly, and ETTm1 and ETTm2, which are sampled every 15 minutes.
  \item Exchange \cite{autoformer}: This dataset collects the daily exchange rates of 8 countries from 1990 to 2016.
  \item ILI \cite{zeng2022}: This dataset represents the proportion of patients diagnosed with influenza-like illness relative to the total number of patients. It includes weekly data from the Centers for Disease Control and Prevention of the United States from 2002 to 2021.
  \item Weather \cite{liu2024itransformer}: This dataset consists of 21 meteorological variables recorded at 10-minute intervals by the Weather Station of the Max Planck Institute for Biogeochemistry in 2020.
  \item ECL (Electricity Consuming Load) \cite{autoformer}: This dataset collects the hourly electricity consumption of 321 clients from 2012 to 2014.
\end{itemize}
The details of the datasets are provided in Table \ref{Table_data}.

\section{D. Backbone Models}
We use five long-term time series forecasting models as our backbone models. These include three Transformer-based models Informer \cite{Informer}, Autoformer \cite{autoformer}, and iTransformer \cite{liu2024itransformer} and two MLP-based models Dlinear \cite{zeng2022}, RAFT \cite{han2025retrieval}. A brief introduction to these models is given as follows. For more details, please refer to their original papers.
\begin{itemize}
  \item Informer: Its core idea is to use a sparse self-attention mechanism based on probabilistic sampling, enabling efficient modeling of long-range dependencies in long sequences and significantly improving forecasting efficiency.
  \item Autoformer: Its core approach is to introduce a recursive decomposition module to extract trend and seasonal components, which are then adaptively modeled using self-attention to enhance forecasting accuracy.
  \item Dlinear: Its core concept is to simplify time series modeling by decomposing the input into trend and residual components, which are then modeled separately with linear layers to achieve both efficiency and competitive performance.
  \item iTransformer: Its core idea is to embed the entire time series into a single token and apply attention across the channel dimension, explicitly capturing inter-channel dependencies for improved modeling capacity.
  \item RAFT: Its key approach is to leverage a retrieval-augmented mechanism that constructs context representations from historically similar segments, enhancing the modeling of complex patterns when combined with multilayer perceptrons.
\end{itemize}

\section{E. The Friedman Test} \label{ref_test}
The Friedman test is a non-parametric statistical method widely used for comparing multiple algorithms across several datasets. It evaluates whether there are significant differences in performance rankings among the algorithms.

The test statistic follows an $F$-distribution with $k-1$ and $(k-1)(N-1)$ degrees of freedom and is computed as:
\begin{equation}
\tau_F = \frac{(N-1)\tau_\chi^2}{N(k-1) - \tau_\chi^2},
\end{equation}

where
\begin{equation}
\tau_\chi^2 = \frac{12N}{k(k+1)}\left(\sum_{i=1}^k r_i^2 - \frac{k(k+1)^2}{4}\right).
\end{equation}
Here, $k$ denotes the number of compared models, $N$ represents the number of datasets, and $r_i$ is the average rank of the $i$-th model across all datasets. A lower average rank indicates better model performance.

If the computed $\tau_F$ exceeds the critical value at a chosen significance level, the null hypothesis---that all models perform equally---is rejected, suggesting statistically significant differences among the algorithms.

As shown in Table \ref{Table_Friedman}, at significance level $\alpha =0.05$, both MSE and MAE clearly reject the hypothesis that all models have the same performance.
\begin{table}[ht]
    \centering
    \renewcommand{\arraystretch}{1.2}
    \caption{Summary of the Friedman Statistics$\tau _{F}$}
     \label{Table_Friedman}
    \begin{tabular}{lcc}
        \hline
        Evaluation metric     & $\tau _{F}$ & Critical value$(\alpha=0.05)$ \\
        \hline
        MSE     & 235.6096       & 1.9135  \\
        MAE     & 242.2829       & 1.9135  \\
        \hline
    \end{tabular}
\end{table}

Further, the Nemenyi post-hoc test determines whether two algorithms differ significantly by comparing their average rank difference to the critical distance:
\begin{equation}
CD = q_\alpha \sqrt{\frac{k(k+1)}{6N}},
\end{equation}
where $q_\alpha$ is the critical value from the Studentized range distribution at significance level $\alpha$ (with $q_{0.05}=3.164$ for $k=10$ models), $k$ is the number of compared algorithms, and $N$ is the number of datasets. The performance difference is considered statistically significant when the observed rank difference exceeds $CD$.

\section{F. More Experimental Results}
This section includes:
\begin{itemize}
  \item F.1 More Results on Running Cost.
  \item F.2 More Results on Ablation Study.
  \item F.3 More Results on Hyperparameter Sensitivity.
  \item F.4 More Visualization of Prediction Results.
  \item F.5 Lookback Window Size.
  \item F.6 Correlation Structure Visualization.
  \item F.7 Comparison with Other Loss Functions.
\end{itemize}
All the experiments are implemented by PyTorch on two NVIDIA RTX A6000 48GB GPUs.

\subsection{F.1 More Results on Running Cost} \label{ref_run}

As shown in Table \ref{Table_eff_ecl}, RI-Loss incurs slightly higher computational overhead than MSE on large-scale datasets, but the additional cost remains manageable. Despite introducing extra training time compared to the traditional MSE, RI-Loss consistently achieves better forecasting accuracy across various models and prediction lengths, demonstrating a favorable trade-off between performance and efficiency.

\begin{table}[ht]
\centering
\fontsize{8}{8}\selectfont
\setlength{\tabcolsep}{3.5pt} 
\caption{Training time (ms/iter) of MSE and RI-Loss with a forecasting horizon of 96 on the Electricity dataset.}
\label{Table_eff_ecl}
\begin{tabular}{c c c c c c}
\toprule
Models & Informer & Autoformer & DLinear & iTransformer \\
\midrule
MSE     & 104.3 & 153.7 & 13.6 & 50.2 \\
RI-Loss & 116.3 & 166.4 & 55.7 & 63.7 \\
\bottomrule
\end{tabular}
\vspace{-13pt}  
\end{table}

\subsection*{F.2 More Results on Ablation Study}  \label{ref_ablation}
Firstly, we provide the calculation formulas for the Pearson correlation coefficient and MAE, while the formulas for RI-Loss and MSE are presented in the main text.

\begin{definition}[Pearson Correlation Coefficient]
For paired sample data $\{(r_i, s_i)\}_{i=1}^n$, the Pearson correlation coefficient $PC(\bm{r}, \bm{s})$ is given by:
\begin{equation}
PC(\bm{r}, \bm{s}) = \frac{\sum_{i=1}^n (r_i - \bar{r})(s_i - \bar{s})}{\sqrt{\sum_{i=1}^n (r_i - \bar{r})^2} \sqrt{\sum_{i=1}^n (s_i - \bar{s})^2}},
\end{equation}
where: $\bm{r} = (r_1, \ldots, r_n)$ and $\bm{s} = (s_1, \ldots, s_n)$ are observed samples and
$\bar{r} = \frac{1}{n}\sum_{i=1}^n r_i$ and $\bar{s} = \frac{1}{n}\sum_{i=1}^n s_i$ are the sample means.
\end{definition}

The Pearson Correlation Coefficient is used to measure the correlation between residual series and noise series: $PC(\bm{Y} - \hat{\bm{Y}}, \bm{\epsilon})$.
 \begin{definition}
For the true series $\boldsymbol{Y}$ and its predicted counterpart $\hat{\boldsymbol{Y}}$, the Mean Absolute Error (MAE) is computed as:
\begin{equation}
\mathrm{MAE} = \frac{1}{H}\|\boldsymbol{Y} - \hat{\boldsymbol{Y}}\|_1 = \frac{1}{H}\sum_{k=1}^H \|\boldsymbol{y}_{t+k} - \hat{\boldsymbol{y}}_{t+k}\|_1.
\end{equation}
\end{definition}

\begin{figure}
\centering
\vspace{-7pt}
\captionsetup[subfigure]{labelformat=empty,skip=2pt}
\hspace{0.05\textwidth}
\subfloat{\includegraphics[width=0.35\textwidth]{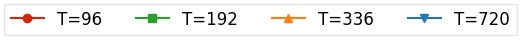}}\\
\subfloat[(a) ETTh1]{\includegraphics[width=0.23\textwidth]{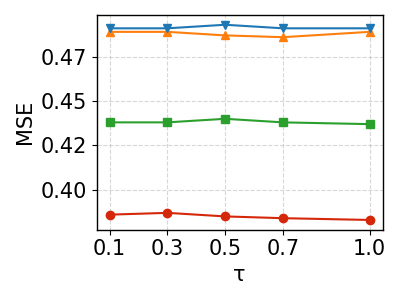}}
\subfloat[(b) ETTm1]{\includegraphics[width=0.23\textwidth]{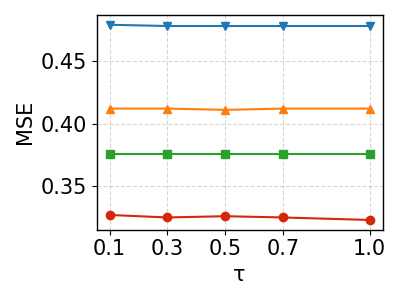}}
\caption{The impact of the hyperparameter $\tau$ on iTransformer based on ETTh1 and ETTm1 datasets.}
\label{Figure_tao}
\end{figure}

We conducted ablation experiments on five backbone models using the ETTh1, ETTm1, and Weather datasets,
using RI-Loss, MAE, MSE, and RI-Loss with HSIC replaced by the Pearson Correlation coefficient (PC) as loss functions. Results for using HSIC alone as the loss function are not included in the table, as this approach leads to severe performance degradation.
The results are shown in Table \ref{Table_ILI}.
The results show that replacing HSIC with PC degrades performance, as PC only captures linear relationships and cannot model nonlinear dependencies in complex time series. Meanwhile, RI-Loss outperforms both MSE and MAE, validating the effectiveness of our design.
Additionally, we found that using MAE as the loss function yields better predictive performance than MSE. This is because MSE is highly sensitive to noise and tends to amplify its effect during optimization, causing the model to overfit noise and degrade overall performance.

\begin{figure}[h]
\centering
\subfloat[ETTm1(MSE-Loss)]{\includegraphics[width=0.21\textwidth]{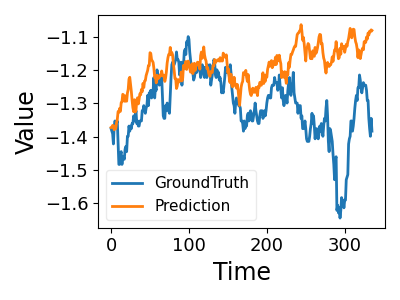}}
\subfloat[ETTm1(RI-loss)]{\includegraphics[width=0.21\textwidth]{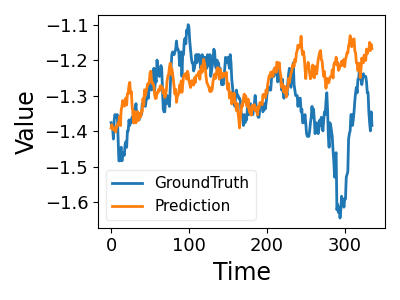}} \\
\subfloat[ETTm2(MSE-Loss)]{\includegraphics[width=0.21\textwidth]{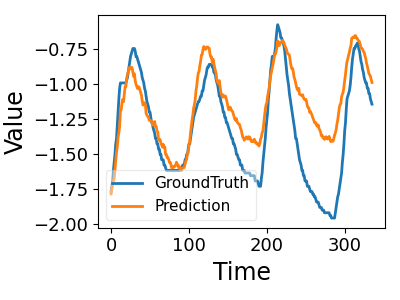}}
\subfloat[ETTm2(RI-loss)]{\includegraphics[width=0.21\textwidth]{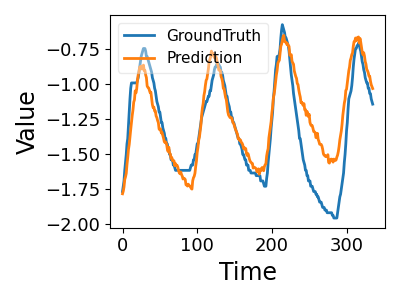}} \\
\subfloat[Weather(MSE-Loss)]{\includegraphics[width=0.21\textwidth]{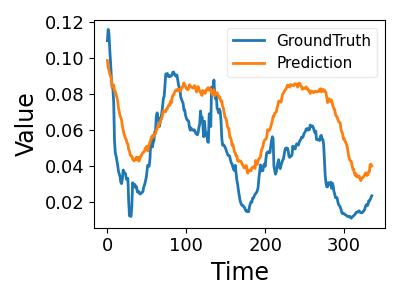}}
\subfloat[Weather(RI-loss)]{\includegraphics[width=0.21\textwidth]{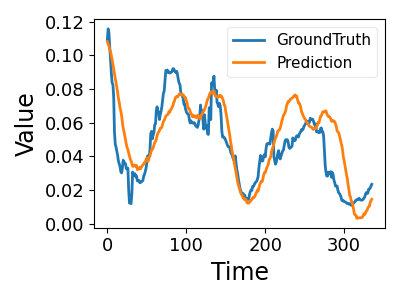}} \\
\caption{Forecasting visualization comparing RI-Loss and MSE loss as objective functions under the input-336-predict-336 settings. Blue lines are the ground truths and orange lines are the model predictions.}
\label{fig_visual}
\vspace{-1em}
\end{figure}

\begin{table*} 
\centering
\footnotesize
\setlength{\tabcolsep}{4pt}
\renewcommand{\arraystretch}{1}
\caption{Comparison of Different Loss Functions Across Five Backbone Models on the ETTh1, ETTm1, and Weather Datasets.}
\label{Table_ILI}
\begin{tabular*}{\hsize}{@{\extracolsep{\fill}}c|c|cccc|cccc|cccc}
\toprule
\multirow{2}{*}{\textbf{Model}} & \multirow{2}{*}{\textbf{Loss}} & \multicolumn{4}{c|}{\textbf{ETTh1}} & \multicolumn{4}{c|}{\textbf{ETTm1}} & \multicolumn{4}{c}{\textbf{Weather}} \\
\cmidrule(r){3-6} \cmidrule(lr){7-10} \cmidrule(lr){11-14}
 & & \textbf{96} & \textbf{192} & \textbf{336} & \textbf{720} & \textbf{96} & \textbf{192} & \textbf{336} & \textbf{720} & \textbf{96} & \textbf{192} & \textbf{336} & \textbf{720} \\
\midrule
\multirow{4}{*}{Informer}
& RI-Loss & 0.875 & \textbf{0.970} & 1.225 & 1.169 & \textbf{0.536} & \textbf{0.610} & \textbf{0.807} & \textbf{0.978} & \textbf{0.479} & \textbf{0.575} & \textbf{0.609} & \textbf{0.637} \\
& MAE & \textbf{0.870} & 1.020 & 1.198 & \textbf{1.159} & 0.553 & 0.622 & 0.873 & 1.024 & 0.485 & 0.579 & 0.626 & 0.641 \\
& MSE & 0.878 & 1.013 & \textbf{1.172} & 1.175 & 0.678 & 0.784 & 1.011 & 1.037 & 0.496 & 0.582 & 0.643 & 0.651 \\
& Pearson+MSE & 0.986 & 1.005 & 1.271 & 1.188 & 0.673 & 0.764 & 1.037 & 1.225 & 0.500 & 0.586 & 0.623 & 0.672 \\
\midrule
\multirow{4}{*}{Autoformer}
& RI-Loss & \textbf{0.440} & 0.474 & 0.522 & 0.544 & \textbf{0.466} & \textbf{0.534} & \textbf{0.480} & \textbf{0.523} & \textbf{0.224} & \textbf{0.283} & \textbf{0.347} & \textbf{0.419} \\
& MAE & 0.453 & \textbf{0.471} & 0.532 & 0.566 & 0.469 & 0.589 & 0.594 & 0.594 & 0.231 & 0.289 & 0.350 & 0.422 \\
& MSE & 0.456 & 0.492 & \textbf{0.506} & \textbf{0.496} & 0.478 & 0.549 & 0.516 & 0.526 & \textbf{0.224} & 0.305 & 0.353 & 0.456 \\
& Pearson+MSE & 0.442 & 0.488 & 0.525 & 0.551 & 0.514 & 0.581 & 0.625 & 0.641 & 0.238 & 0.285 & 0.351 & 0.436 \\
\midrule
\multirow{4}{*}{DLinear}
& RI-Loss & \textbf{0.346} & \textbf{0.404} & \textbf{0.443} & \textbf{0.471} & \textbf{0.294} & \textbf{0.333} & \textbf{0.369} & \textbf{0.423} & \textbf{0.143} & \textbf{0.184} & \textbf{0.234} & \textbf{0.309} \\
& MAE & 0.367 & 0.407 & 0.472 & 0.476 & 0.296 & 0.334 & 0.371 & 0.425 & 0.178 & 0.217 & 0.261 & 0.322 \\
& MSE & 0.384 & 0.443 & 0.447 & 0.472 & 0.302 & 0.337 & 0.371 & 0.426 & 0.144 & 0.188 & 0.240 & 0.317 \\
& Pearson+MSE & 0.374 & 0.407 & 0.499 & 0.483 & 0.310 & 0.344 & 0.378 & 0.432 & 0.183 & 0.223 & 0.281 & 0.334 \\
\midrule
\multirow{4}{*}{iTransformer}
& RI-Loss & \textbf{0.382} & \textbf{0.437} & \textbf{0.490} & \textbf{0.490} & \textbf{0.322} & \textbf{0.376} & \textbf{0.412} & \textbf{0.478} & \textbf{0.169} & 0.219 & \textbf{0.278} & 0.357 \\
& MAE & 0.383 & 0.439 & 0.495 & 0.492 & 0.323 & 0.379 & 0.417 & 0.479 & 0.172 & \textbf{0.218} & 0.281 & \textbf{0.355} \\
& MSE & 0.387 & 0.441 & 0.491 & 0.508 & 0.342 & 0.383 & 0.418 & 0.487 & 0.174 & 0.224 & 0.283 & 0.359 \\
& Pearson+MSE & 0.396 & 0.446 & 0.498 & 0.525 & 0.348 & 0.387 & 0.423 & 0.490 & 0.175 & 0.221 & 0.288 & 0.363 \\
\midrule
\multirow{4}{*}{RAFT}
& RI-Loss & \textbf{0.378} & \textbf{0.422} & 0.460 & \textbf{0.462} & \textbf{0.316} & \textbf{0.357} & \textbf{0.385} & 0.447 & \textbf{0.189} & \textbf{0.234} & \textbf{0.283} & \textbf{0.358} \\
& MAE & 0.379 & 0.426 & 0.465 & 0.468 & 0.320 & 0.359 & 0.388 & 0.449 & 0.191 & 0.236 & 0.285 & 0.368 \\
& MSE & 0.387 & 0.423 & \textbf{0.458} & 0.463 & 0.329 & 0.363 & 0.391 & \textbf{0.444} & \textbf{0.189} & 0.239 & 0.291 & 0.366 \\
& Pearson+MSE & 0.399 & 0.433 & 0.464 & 0.466 & 0.333 & 0.367 & 0.393 & 0.450 & 0.197 & 0.247 & 0.299 & 0.374 \\
\bottomrule
\end{tabular*}
\end{table*}

\begin{table*} 
\centering
\caption{The prediction performance of four backbone models using different $\lambda$ with RI-Loss on the ETTh2, ETTm1 and Weather datasets.}
\label{Table_lambda}
\setlength{\tabcolsep}{0pt}  
\renewcommand{\arraystretch}{1}
\small
\begin{tabularx}{\textwidth}{c@{\hskip 2pt} *{4}{>{\centering\arraybackslash}X@{\hskip 2pt}}}
\toprule
Dataset & Informer & Autoformer & DLinear & RAFT \\
\midrule
ETTm1 &
\raisebox{-.5\height}{\includegraphics[width=\linewidth, trim=5 5 5 5, clip]{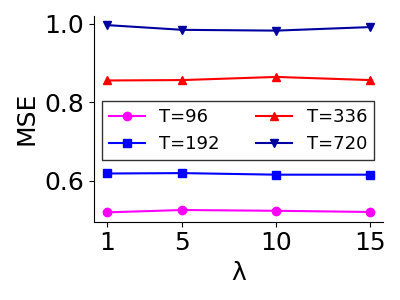}} &
\raisebox{-.5\height}{\includegraphics[width=\linewidth, trim=5 5 5 5, clip]{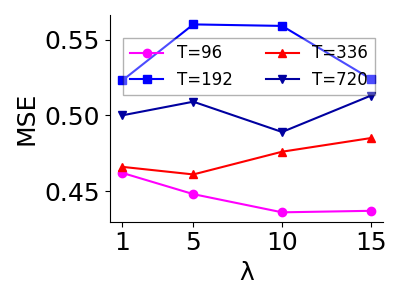}} &
\raisebox{-.5\height}{\includegraphics[width=\linewidth, trim=5 5 5 5, clip]{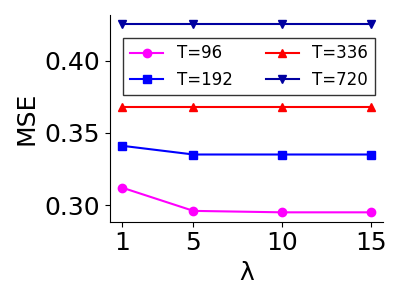}} &
\raisebox{-.5\height}{\includegraphics[width=\linewidth, trim=5 5 5 5, clip]{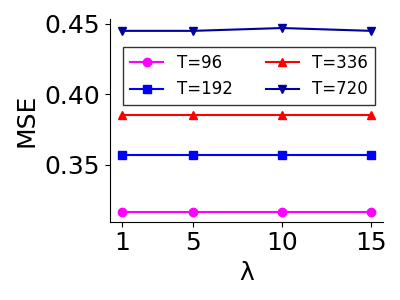}} \\
\midrule
ETTh2 &
\raisebox{-.5\height}{\includegraphics[width=\linewidth, trim=5 5 5 5, clip]{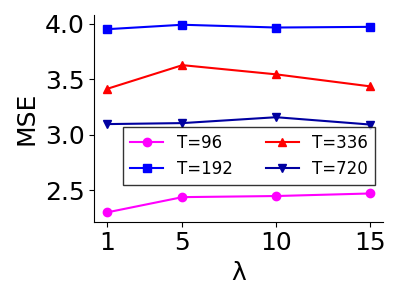}} &
\raisebox{-.5\height}{\includegraphics[width=\linewidth, trim=5 5 5 5, clip]{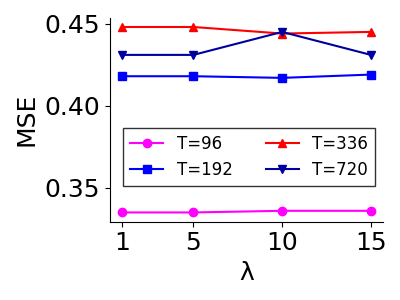}} &
\raisebox{-.5\height}{\includegraphics[width=\linewidth, trim=5 5 5 5, clip]{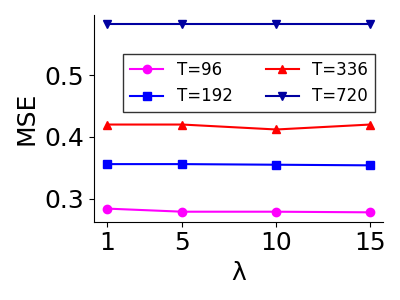}} &
\raisebox{-.5\height}{\includegraphics[width=\linewidth, trim=5 5 5 5, clip]{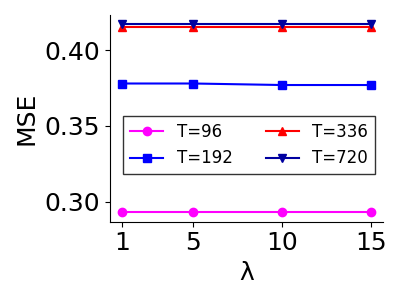}} \\
\midrule
Weather &
\raisebox{-.5\height}{\includegraphics[width=\linewidth, trim=5 5 5 5, clip]{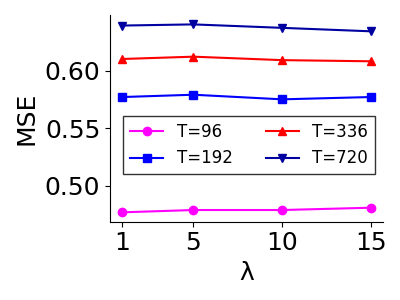}} &
\raisebox{-.5\height}{\includegraphics[width=\linewidth, trim=5 5 5 5, clip]{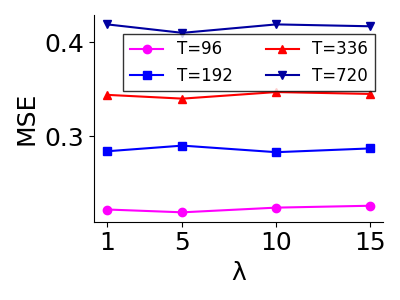}} &
\raisebox{-.5\height}{\includegraphics[width=\linewidth, trim=5 5 5 5, clip]{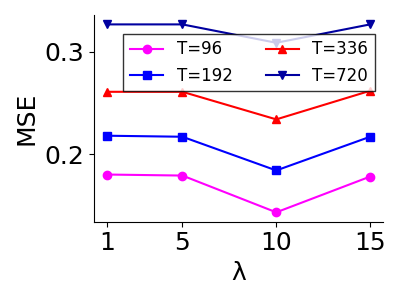}} &
\raisebox{-.5\height}{\includegraphics[width=\linewidth, trim=5 5 5 5, clip]{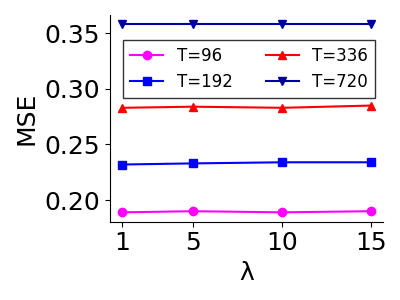}} \\
\bottomrule
\end{tabularx}
\end{table*}

\begin{figure*} 
\centering
\subfloat[DLinear-ETTh2]{\includegraphics[width=0.24\textwidth]{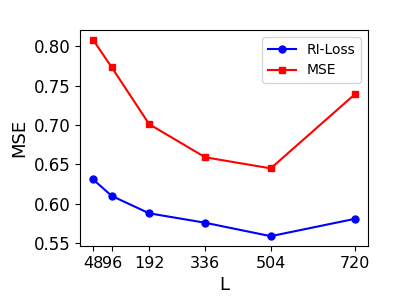}}
\subfloat[DLinear-ETTm2]{\includegraphics[width=0.24\textwidth]{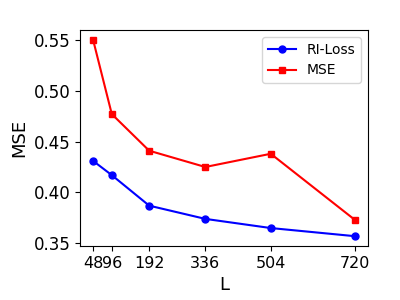}}
\subfloat[Autoformer-ETTh2]{\includegraphics[width=0.24\textwidth]{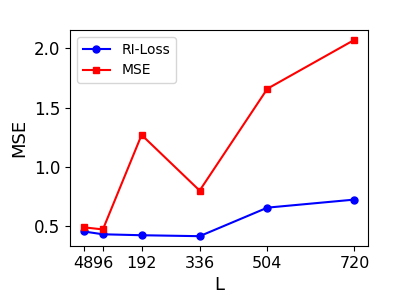}}
\subfloat[Autoformer-ETTm2]{\includegraphics[width=0.24\textwidth]{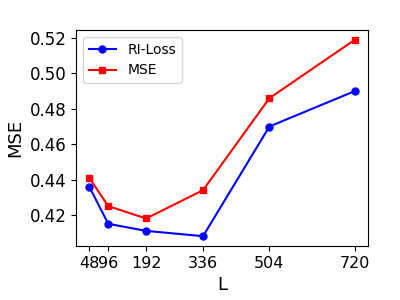}} \\
\subfloat[Informer-ETTh2]{\includegraphics[width=0.24\textwidth]{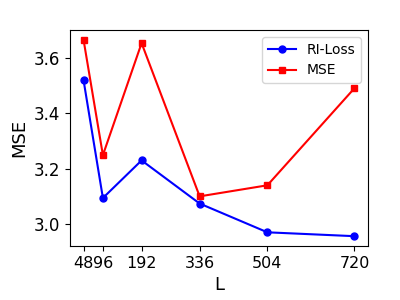}}
\subfloat[Informer-ETTm2]{\includegraphics[width=0.24\textwidth]{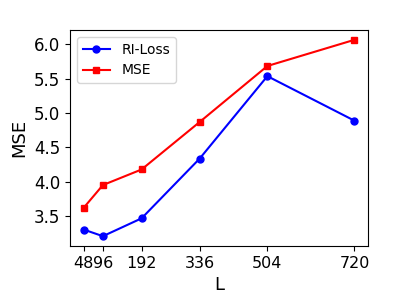}}
\subfloat[iTransformer-ETTh2]{\includegraphics[width=0.24\textwidth]{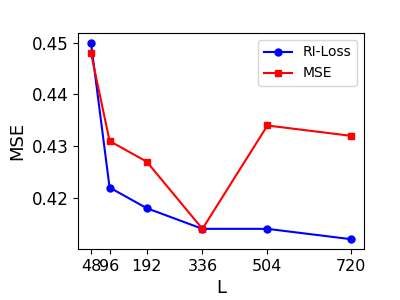}}
\subfloat[iTransformer-ETTm2]{\includegraphics[width=0.24\textwidth]{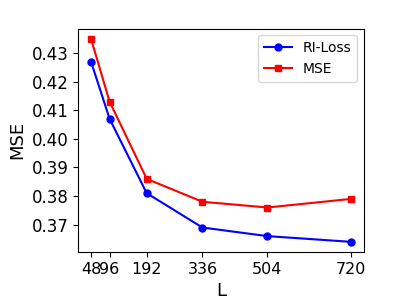}} \\
\vspace{10pt}
\caption{The MSE results for long-term forecasting ($H=720$) on the ETTh2 and ETTm2 datasets using different lookback window sizes.}
\label{figure_pred_len}
\end{figure*}

\begin{figure}
\vspace{-7pt}
\centering
\subfloat[Groundtruth]{\includegraphics[width=0.15\textwidth]{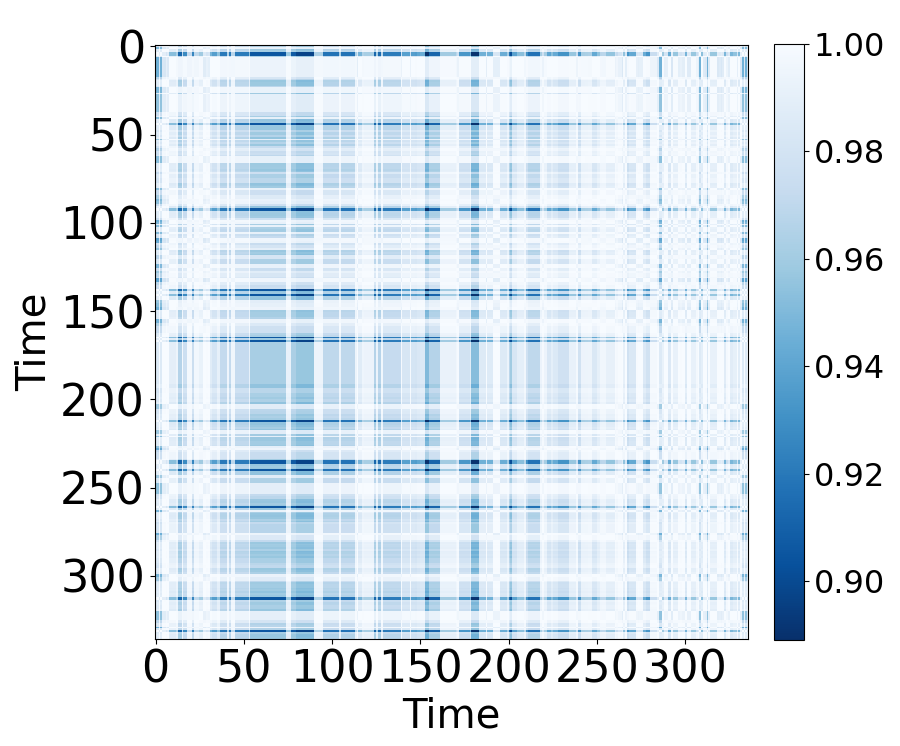}}
\subfloat[MSE-loss]{\includegraphics[width=0.15\textwidth]{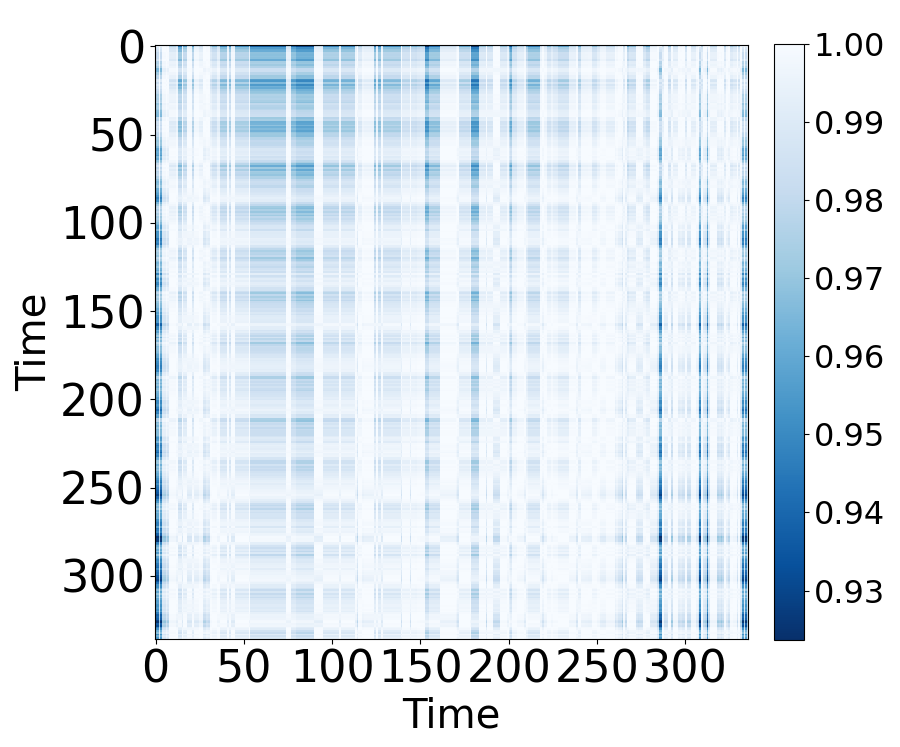}}
\subfloat[RI-Loss]{\includegraphics[width=0.15\textwidth]{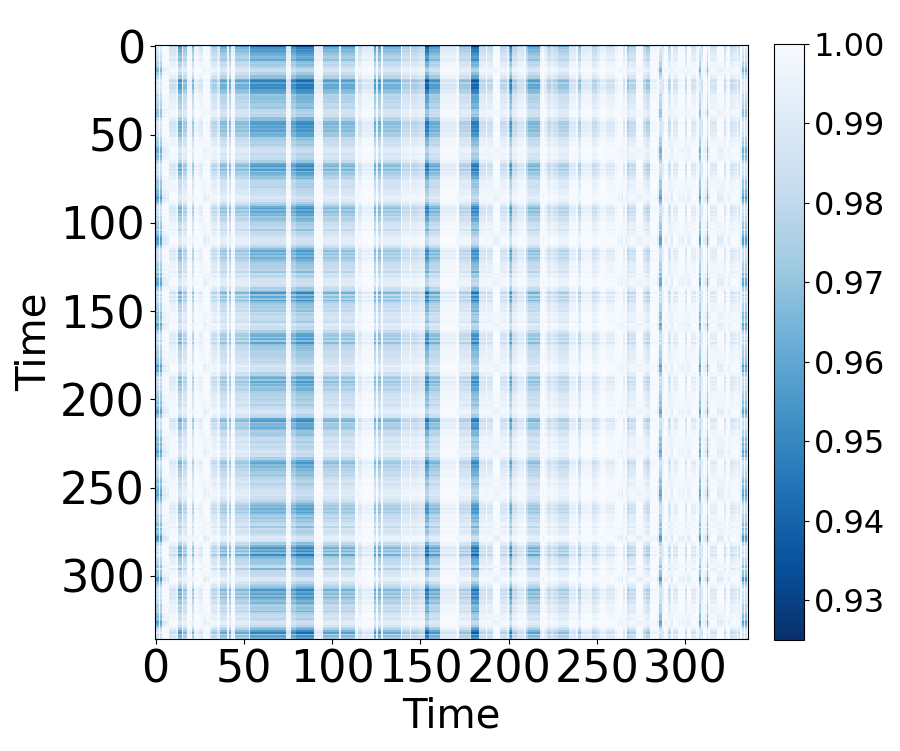}}
\caption{Visualization of temporal dependencies in the ETTh2 dataset.}
\label{Figure_relitu}
\vspace{-5pt}
 \end{figure}

\begin{table*}[ht]
\centering
\caption{Forecasting performance (MSE / MAE) on ETTh1, ETTm1 and Weather datasets. Lower is better.}
\setlength{\tabcolsep}{3pt} 
\renewcommand{\arraystretch}{1.2}  
\begin{tabular}{c|c|cccc|cccc|cccc}
\toprule
\multicolumn{2}{c|}{\textbf{Dataset}}
& \multicolumn{4}{c|}{\textbf{ETTh1}}
& \multicolumn{4}{c|}{\textbf{ETTm1}}
& \multicolumn{4}{c}{\textbf{Weather}} \\
\midrule
\multicolumn{2}{c|}{\textbf{Forecast Length}}
& 96 & 192 & 336 & 720 & 96 & 192 & 336 & 720 & 96 & 192 & 336 & 720 \\
\midrule
\multirow{2}{*}{\textbf{MSE}}
& MSE & 0.387 & 0.441 & 0.491 & 0.508 & 0.342 & 0.383 & 0.418 & 0.487 & 0.174 & 0.224 & 0.283 & 0.359 \\
& MAE & 0.405 & 0.436 & 0.462 & 0.493 & 0.377 & 0.396 & 0.418 & 0.457 & 0.214 & 0.257 & 0.300 & 0.350 \\
\midrule
\multirow{2}{*}{\makecell[c]{\textbf{TILDE-Q}\\(2022)}}
& MSE & 0.386 & 0.434 & 0.477 & 0.521 & 0.336 & 0.382 & 0.418 & 0.483 & 0.172 & 0.221 & 0.279 & \textbf{0.356} \\
& MAE & 0.401 & 0.429 & 0.454 & 0.502 & 0.366 & 0.392 & 0.416 & 0.453 & 0.208 & 0.252 & 0.296 & 0.346 \\
\midrule
\multirow{2}{*}{\makecell[c]{\textbf{FreDF}\\(2024)}}
& MSE & \textbf{0.381} & \textbf{0.433} & \textbf{0.471} & \textbf{0.487} & 0.334 & 0.380 & 0.416 & 0.486 & 0.170 & 0.222 & 0.279 & 0.357 \\
& MAE & 0.396 & \textbf{0.426} & \textbf{0.446} & 0.481 & 0.367 & 0.393 & 0.414 & 0.451 & 0.210 & 0.254 & 0.297 & 0.347 \\
\midrule
\multirow{2}{*}{\textbf{RI-Loss}}
& MSE & 0.382 & 0.437 & 0.490 & 0.490 & \textbf{0.322} & \textbf{0.376} & \textbf{0.412} & \textbf{0.478} & \textbf{0.169} & \textbf{0.219} & \textbf{0.278} & 0.357 \\
& MAE & \textbf{0.395} & 0.427 & 0.458 & \textbf{0.477} & \textbf{0.350} & \textbf{0.378} & \textbf{0.403} & \textbf{0.441} & \textbf{0.201} & \textbf{0.248} & \textbf{0.291} & \textbf{0.345} \\
\bottomrule
\end{tabular}
\label{tab:loss}
\end{table*}

\subsection{F.3 More Results on Hyperparameter Sensitivity}  \label{ref_lambda}

We evaluate the impact of the hyperparameter $\lambda$ in RI-Loss on prediction performance using the ETTh2, ETTm1, and Weather datasets, with $\lambda \in \{1, 5, 10, 15\}$. As shown in Table \ref{Table_lambda}, the performance remains stable across settings, and the maximum error is still lower than that with MSE loss. Notably, $\lambda$ has minimal effect in long-term forecasting.

In addition, we evaluate the impact of the hyperparameter $\tau$ using the iTransformer model when $\lambda$ is set to 10. We test $\tau \in \{0.1, 0.3, 0.5, 0.7, 1\}$, as shown in Figure \ref{Figure_tao}. The results demonstrate that the model performance remains stable across different values of $\tau$.

\subsection{F.4 More Visualization of Prediction Results}  \label{ref_vis}

To evaluate the prediction quality of the Mean Squared Error (MSE) loss and our proposed RI-Loss, we use DLinear as the backbone model to visualize sample predictions on the ETTm1, ETTm2, and Weather datasets. The results are shown in Figure \ref{fig_visual}.

\subsection{F.5 Lookback Window Size}

In long-term time series forecasting tasks, it is generally expected that more powerful models will perform better with larger lookback windows~\cite{xu2024fits,zeng2022,li2023r}. However, in practice, expanding the lookback window does not always lead to improved performance. As the window size increases, the model must process more features, which can lead to higher complexity and lower performance.

Figure \ref{figure_pred_len} illustrates the long-term forecasting capabilities of three baseline models with varying lookback window sizes \{48, 96, 192, 336, 504, 720\}. From the figures, we observe the following:

(1) DLinear and iTransformer benefit more significantly as the lookback window increases, especially when trained with our loss function, which further amplifies this advantage.

(2) For Informer and Autoformer, our loss function helps the models better leverage moderately extended lookback windows, resulting in improved forecasting performance.

(3) There is a limit to the performance gains achieved by increasing the lookback window. If there is a distribution shift in the dataset, a longer window may introduce irrelevant or outdated information, which can degrade model performance.

\subsection{F.6 Correlation Structure Visualization}

In Figure \ref{Figure_relitu}, we present a visualization of the temporal dependencies learned by the Dlinear model using MSE-Loss and RI-Loss. Panel (a) shows the correlation heatmap between future and historical real data. Panels (b) and (c) display the correlation heatmaps between the predicted and historical values when the model uses MSE-Loss and RI-Loss, respectively. It is easy to see that when the model uses MSE-Loss, it captures some irrelevant dependencies. In contrast, when using RI-Loss, the correlation between the predicted and historical values more closely aligns with the actual correlation, demonstrating a better alignment with the temporal relationships.

\subsection{F.7 Comparison with Other Loss Functions}

We evaluate the proposed RI-Loss in comparison with advanced loss functions. TILDE-Q \cite{lee2022tildeq} enhances the model’s sensitivity to temporal shape patterns through a transformation-invariant regularization term, thereby improving structural consistency in predictions. FreDF \cite{wang2025fredf}, on the other hand, leverages frequency-domain comparisons to reduce reliance on noisy labels and emphasizes spectral structure alignment. Table \ref{tab:loss} presents the forecasting performance of different loss functions across multiple benchmark datasets and prediction lengths. The results show that RI-Loss consistently achieves the lowest MSE and MAE in most settings, outperforming both TILDE-Q and FreDF. This performance gain can be attributed to RI-Loss’s ability to better capture critical temporal patterns and to effectively suppress the adverse impact of noise during optimization.

\end{document}